%% file: main.tex
\documentclass[final,12pt]{sty/colt2024bare} 

\input{defs}

    \title[Noise misleading rotation invariant algorithms] 
    {Noise misleads rotation invariant algorithms on sparse targets}
\usepackage{wrapfig}
\usepackage{forest}
\usepackage{framed}
\usepackage[font=footnotesize,labelfont=bf]{caption}
\usepackage{times}
\usepackage{bm}
\usepackage{xfrac}

\coltauthor{\Name{Manfred K. Warmuth} 
$\;\mathrm{and}\;$
  \Name{Ehsan Amid} 
  \Email{manfred,eamid@google.com,}
  \addr Google Inc.}

\coltauthor{\Name{Manfred K. Warmuth} \Email{manfred@google.com}\\
\addr Google Inc.\\
\Name{Wojciech Kot{\polishl}owski} \Email{wkotlowski@cs.put.poznan.pl}\\
\addr Institute of Computing Science, Poznań University of Technology, Poznań, Poland\\
\Name{Matt Jones} \Email{mcj@colorado.edu}\\
\addr University of Colorado Boulder, Colorado, USA\\
\Name{Ehsan Amid} \Email{eamid@google.com}\\
\addr Google Inc.}

\newcommand{\commentout}[1]{}

\begin{document}

\maketitle

\begin{abstract}%
It is well known that the class of rotation invariant
algorithms are suboptimal even for learning sparse linear problems 
when the number of examples is below the ``dimension'' of the problem.
This class includes any gradient descent trained
neural net with a fully-connected input layer 
    (initialized with a rotationally symmetric distribution).
The simplest sparse problem is learning a single feature
out of $d$ features. In that case the classification error or regression loss
grows with $1-\sfrac kd$ where $k$ is the number of examples seen.
These lower bounds become vacuous when the number
of examples $k$ reaches the dimension $d$.

Nevertheless we show that when noise is added to this sparse linear problem, rotation invariant algorithms are still suboptimal after seeing $d$ or more examples.
We prove this via a lower bound for the Bayes optimal algorithm on a rotationally symmetrized problem.
We then prove much lower upper bounds on the same problem for simple non-rotation invariant algorithms.
Finally we analyze the gradient flow trajectories of many
standard optimization algorithms in some simple cases and show how 
they veer toward or away from the sparse targets.

We believe that our trajectory categorization will be useful in designing algorithms 
that can exploit sparse targets and
our method for proving lower bounds will be crucial for
analyzing other families of algorithms that admit different classes of invariances.
\end{abstract}

Any gradient descent trained neural net with a fully-connected input layer is rotation invariant when initialized with a rotationally symmetric distribution. 
The reason is that if the input instances are rotated, then
weights rotate in the same way and therefore the dot products
(logits) feeding into the non-linearities of the first layer remain unchanged.
It is well known that rotation invariant
algorithms are fundamentally inferior when learning certain
sparse linear problems \citep{span,ng,arora-conv,spindly}.
These bounds are surprising because they hold for such a
general class of algorithms and are complemented by
simple non-rotation invariant algorithms that need
exponentially fewer examples to achieve the same accuracy. 
However, these lower bounds are not practically relevant 
because they all assume that the number of
examples is less than the dimension of the problem. 

The simplest sparse linear problem is learning one out
of $d$ features. When the instances are for example the $d$
orthogonal rows of a Hadamard matrix and a rotation
invariant algorithm is given $k\le d$ training instances%
\footnote{Learning is harder for sampling with replacement.
However lower bounds are stronger for sampling without replacement
and the boundary of getting a full rank instance set is clearer.
We present bounds for both types of models.}
then its generalization error on all $d$ instances is at
least $1-\sfrac kd$. There are classification \citep{ng,arora-conv}
and regression versions of this problem \citep{span,spindly}
but the lower bound is essentially the same, proven with
different techniques.
However, the lower bounds become vacuous when $k\ge d$, which is the case in most relevant settings.
For example linear regression (which is rotation invariant)
finds the unique consistent feature after seeing a
noise-free training set of full rank. 
In contrast this paper provides lower bounds on the loss 
of any rotation invariant algorithm
when the sample size (without replacement) exceeds the
input dimension of the problem.
Now any weight vector (including sparse ones) can be expressed as a linear combination of the instances. 
The key point of our paper is that, when the noise is added to the 
sparse target, rotation invariant algorithms still produce poor solutions compared to simple non-invariant algorithms.

Our lower bound technique creates a Bayesian setup
where the learning is presented with a randomly rotated
version of the input instances. We prove
a lower bound on the Bayes-optimal algorithm for this case,
and we prove any rotational invariant algorithm on the
original problem can do no better than the Bayes-optimal
algorithm on the randomly rotated problem.%
\footnote{This proof technique is interesting in its own right
and is different from the regression techniques used for the
under-constrained case \citep{spindly}.}
We then show that there are trivial non-rotationally invariant algorithms that can learn
noisy sparse linear much more efficiently: multiplicative updates on a linear neuron or 
gradient descent on a `spindly' two-layer linear net
where every linear weight $w_i$ is replaced by a product of parameter $u_iv_i$ 
\begin{wrapfigure}{r}{0.155\textwidth}
    \vspace{-9mm}
    \footnotesize
    \begin{center}
\begin{forest}
for tree={circle,draw,s sep=11pt, scale=.40,l=2}
[,
    [ [ ] ]
    [ [ ] ]
    [  ,edge label={node[midway,left=-1mm] {$u_i$}}
       [ ,edge label={node[midway,left=-1mm] {$v_i$}}
         ] ]
    [  [ ] ]
    [  [ ] ]
]
\end{forest}
       \end{center}
       \vspace{-6mm}
       \caption{Spindly.}
       \vspace{-0.7cm}
       \label{f:spindly}
\end{wrapfigure}
(Figure \ref{f:spindly}).
Finally we also prove the same upper bound for a novel ``priming method'' that was conjectured to
learn sparse linear problems in
\citep{priming}: First find the linear least squares solution $\w$ 
and then after reweighing the $i$th feature with $w_i$,
apply tuned ridge regression.

We are not interested in linear problems per se, 
but any fancy non-linear model should at least be able to handle the linear case.
We visualize how the performance gap between rotationally invariant 
and non-invariant algorithms arises by computing their trajectories in weight space 
over the course of learning. 
Previous work has focused on how rotational symmetry can be
broken adding a L1 penalty to the loss \citep{ng,arora-conv}.
This changes the minimum to which the learning algorithms converge to.
Here we focus on symmetry breaking due to the differences
in the iterative learning algorithm, 
while keeping the loss function fixed (as in \citet{arora-conv}).
So all algorithms must converge to the same minimum but follow trajectories 
that differ in how close they get to the sparse target.
The optimal early stopping point is the trajectory point that is closest to the target.

We observe that when the number of examples is larger than
the input dimension and the target is sparse linear, then multiplicative updates, 
spindly networks, and priming all produce trajectories that pass extremely close to the sparse target before finally veering away 
and learning the noise.
Rotation invariant algorithms cannot produce such a bias, even when the input is unbalanced (anisotropic covariance) provided it is drawn from a rotationally symmetric distribution.
In essence, they learn the signal and the noise at the same rate.
Interestingly, adaptive learning rate algorithms such as Adagrad \citep{adagrad} and Adam \citep{adam} show the opposite bias, producing trajectories that are curved away from sparse solutions and learn the noise before the signal.

We arrive at these observations by deriving closed-form solutions for the 
continuous time (gradient flow) weight trajectories of each
algorithm for sparse noisy regression problems. We found
forms for the isotropic covariance case but also visualize the
trajectories of the anisotropic case.
We show how the behavior of the different algorithms 
can be understood from a geometric analysis that recasts
mirror descent and other preconditioned gradient methods as
gradient descent under an altered geometry, 
where the preconditioner acts as a Riemannian metric.

Finally we make some preliminary experimental observations for nonlinear models
by contrasting gradient descent training on a simple
three layer feed forward neural net when each logit on the first layer is connected to all inputs
(fully connected) versus connected to all inputs via
the spindly network (Figure \ref{f:spindly}).
Our experiments show in case b, gradient descent
produces sparse solutions for the input layer, is much
less confused by additional noisy features and can learn
quickly from informative features.

\subsubsection*{Additional related work:}
There is a long history of contrasting the generalization
ability of additive versus multiplicative updates (see e.g.
\cite{percwinn,eg}). 
Surprisingly there is a connection between both update families
rooted in the observation that when the weights are products of parameters,
then gradient descent is biased towards sparsity \citep{srebro1,depthnat}. 
More precisely \citep{regretcont}, the gradient flow of multiplicative updates on a
linear neuron equals the gradient flow of
gradient descent on the spindly network of Figure \ref{f:spindly}.

In this paper we prove lower bounds for sparse noisy regression problems.
It might also be possible to obtain lower bounds for the classification setting 
when the number of examples exceeds the VC dimension, building on the setup of \citep{ng,arora-conv}.
However for technical reasons the additional loss in the VC dimension / Rademacher complexity
has a square root term that seems to make this
bounding methodology non-optimal in the noisy case.
Our upper bounds for the spindly network of Figure
\ref{f:spindly} have also been proven using the R.I.P.
assumption \citep{optimal}. For completeness we
gave a self-contained proof in the appendix that also
works for multiplicative updates and their approximations.

Note that we have hardness results for a fixed noisy sparse linear problem: 
First feature plus noise. 
We prove that the gap in performance is due to rotation invariance.
So we avoid using staircase, cork screw or cryptographically hard functions \citep{abbe}
for proving lower bounds. Also in \cite{distr-spec-shamir},
the target class is much more complicated and the hard input distribution is algorithm specific
and therefore no fixed problem is given that is hard for all algorithms
in the invariance class.

\section{The lower bound method}
\subsection{Rotation invariance and problem setup}
An \emph{example} $(\x,y)$ is a $d$-dimensional vector, followed by a real-valued label
$y \in \mathbb{R}$. We specify a training set as a tuple $(\underset{n,
d}{\X},\underset{n}{\y})$ containing $n$ training examples, where the rows of
\emph{input matrix} $\X$ are the $n$ (transposed) input vectors and the
\emph{target} $\y$ is a vector of their labels. 

A \emph{learning algorithm} is a mapping which,
given the training set $(\X,\y)$, produces a real-valued \emph{prediction function}
$\mathbb{R}^d \ni \x \mapsto \yh(\x | \X,\y) \in \mathbb{R}$.
An algorithm is called \emph{rotation invariant}
\citep{span,spindly} if for any orthogonal matrix $\underset{d,d}{\U}$ and any input $\x\in\mathbb{R}^d$:
\vspace{-3mm}
\begin{equation}
\yh(\U \x\, | \X \U^\top,\y) = \yh(\x\, | \X ,\y)
\label{eq:def_rot_inv}
\end{equation}
In other words, the prediction $\yh(\x | \X,\y)$ remains the same if we rotate both $\x$ and all examples from $\X$ by the same orthogonal matrix $\U$.
If the algorithm is randomized, based on an internal random variable $Z$,\footnote{For example, in neural networks $Z$ would correspond to a random initialization of the parameter vector.}
then $\yh(\x_{\rm{te}} | \X, \y)$ is a random variable given by some function $f_{\x_{\rm{te}},\X,\y}(Z)$, 
and the equality sign in equation \eqref{eq:def_rot_inv}
should be interpreted as ``identically distributed''.

Our lower bounds in sections \ref{s:gen_lower_bound}-\ref{s:lower_bounds} 
hold for any rotation invariant algorithm.
In particular, \citet{spindly} have shown that any neural network with a fully-connected input layer (and arbitrary remaining layers), in which the weights in the input layer are initialized randomly with a rotation invariant distribution (e.g. i.i.d. Gaussians) and are trained by gradient descent, is rotation invariant and is thus subject to our lower bound.
The reason is that such a network can be written as $f(\w_1\cdot\x,\w_2\cdot\x,\ldots,\w_h\cdot\x,\bm{\theta})$, and 
the gradient $\nabla_{\w_i}f$  for unit $i$ in the first hidden layer is equal to the instance $\x$ times a scalar that depends on $\x$ only via $\w_i\cdot\x$, and updating of the later layer weights $\bm{\theta}$ depends on the input only via $\w_i\cdot\x$ (i.e., via the computation in the first layer).
Therefore it is easy to show by induction on $t$, 
that rotating all instances results in the same rotation of
all $\w_{i,t}$. Thus the rotation has no effect on the dot
products $\x\cdot\w_i$ computed at the input layer.
In contrast, learning with $f((\u_1\odot\v_1)\cdot\x,
\ldots, (\u_h\odot\v_h)\cdot\x, \bm{\theta})$, 
where the parameters $\w_i$ connected to the input $\x$ are replaced by $\u_i\odot\v_i$
(``spindlified''), is not rotation invariant.

\subsection{Lower bounds for rotation invariant algorithms}
\label{s:gen_lower_bound}

Our method for proving lower bounds builds on the following observation: Given any rotation invariant algorithm and any learning problem, the algorithm will achieve the same loss on all rotated versions of that problem. We can therefore consider a Bayesian setting where the problem is sampled uniformly from all rotated versions, and the optimal solution provides a lower bound on the loss of the algorithm. Intuitively, being rotation invariant forces an algorithm to be agnostic over all possible rotations of the problem, and hedging its bets in this way prevents it from excelling at any specific problem instance. In Section \ref{s:lower_bounds} we apply this reasoning to linear regression to show that a rotation invariant algorithm cannot efficiently learn sparse solutions, because it must be equally efficient at finding any other solution (including rotated, non-sparse ones).

Formally, let the learning problem be defined by 
(a) an input distribution $p_{\textrm{in}}(\Xfull)$ with the input matrix $\underset{n+1,d}{\Xfull} = [\X, \Xte^\top]$ consisting of the training matrix $\X$ and the test example $\Xte$
(b) an observation model $q(\yfull | \Xfull)$ which gives the joint conditional distribution over $n$ training outcomes and a test outcome, $\underset{n+1}{\yfull} = [\y, \yte]$, 
and 
(c) 
a loss function $\mathcal{L}(\yh, y)$. 
We assume the input distribution is rotationally symmetric, meaning $p_{\textrm{in}}(\Xfull) = p_{\textrm{in}}(\Xfull\U^\top)$ for any orthogonal $\U$.
The task of any algorithm will be to produce predictions $\yh(\Xte | \X, \y)$ to minimize the loss on the test outcomes, $\mathcal{L}(\yh,\yte)$.
Note that this setup allows arbitrary conditional dependencies among observations (not just iid problems), including dependencies between the training and the test sets.

For any orthogonal $\underset{d,d}{\U}$, define the rotated observation model as
$\;q_{\U}(\yfull|\Xfull) = q(\yfull|\Xfull\U^{\top})$.
Now define a new learning problem by first sampling $\U$
uniformly (under the Haar measure $p_{\textrm H}$) and then generating observations according to $q_{\U}$. This is equivalent to a symmetrized observation model  $\mathring{q}$ that is a mixture over all $q_{\U}$:
\begin{align*}
\mathring{q}(\yfull | \Xfull) = \int q_{\U}(\yfull | \Xfull) {\textrm d}p_{\textrm H}(\U)
\end{align*}
The Bayes optimal prediction can be expressed by computing a posterior over $\U$ and integrating expected loss over this posterior:
\begin{align*}
\yh^{\star}(\Xte | \X,\y) = \arg\min_{\yh} \int
    \mathbb{E}_{\yte\sim
    q_{\U}(\cdot\vert\Xfull, \y)}[\mathcal{L}(\yh,\yte)]
    p(\U\vert\Xfull,\y)
    {\textrm d}\U.
\end{align*}
Thus $\mathring{q}$ is difficult, especially for large $d$, because equal prior probability must be given to all possible rotations. We define the optimal expected loss on this problem as
\begin{align*}
L_{\mathcal{B}}(\mathring{q}) = \mathbb{E}_{\Xfull\sim
    p_{\textrm{in}}, \yfull\sim\mathring{q}(\cdot|\Xfull)}[\mathcal{L}(\yh^{\star}(\Xte | \X, \y),\yte)].
\end{align*}
Our first result is that the performance of any rotation
invariant algorithm on the original problem, defined by
$q$, is lower bounded by $L_{\mathcal{B}}(\mathring{q})$.

\begin{theorem}
\label{t:gen_lower_bound}
Given a rotationally symmetric $p_{\textrm{in}}(\Xfull)$, an observation model $q(\yfull|\Xfull)$, a loss function $\mathcal{L}$, and a rotationally invariant learning algorithm $\yh_n(\cdot | \X,\y)$, define the expected loss
\begin{align*}
L_{\yh}(q) = \mathbb{E}_{\Xfull \sim p_{\textrm{in}}, \yfull \sim q(\cdot|\Xfull), Z}[\mathcal{L}(\yh(\Xte | \X,\y), \yte)],
\end{align*}
This loss is bounded by $L_{\yh}(q) \ge L_{\mathcal{B}}(\mathring{q})$.
\end{theorem}

The proof in Appendix~\ref{a:thm1-proof} shows any rotationally invariant algorithm will satisfy $L_{\yh}(q_{\U}) = L_{\yh}(q)$ for all $\U$. This implies $L_{\yh}(q) = L_{\yh}(\mathring{q}) \le L_{\mathcal{B}}(\mathring{q})$.

As we show in this paper, a consequence of Theorem \ref{t:gen_lower_bound} is that rotational invariance prevents efficient learning of problems characterized by properties that are not rotationally invariant, such as sparsity. Algorithms with inductive biases for such properties are necessarily not rotation invariant and can give dramatically better performance.
Although we have stated the theorem in terms of rotational
invariance, it is easily extended to other transformation
groups $\mathcal{T}$ on the input (by replacing $\U$ with
elements of $\mathcal{T}$ and requiring $p_{\textrm{in}}$ to be symmetric under $\mathcal{T}$). 
For example, natural gradient descent (NGD) is often touted for being invariant to arbitrary smooth reparameterization \citep{Amari98}, but this invariance comes at a cost of being unable to efficiently learn in environments that are not invariant in this way.
In particular, since spindly and fully connected networks are related by smooth reparameterization, 
NGD performs equivalently on both (see discussion in \citep{depthnat}).
Thus the lower bounds we prove apply to NGD
on the network in Figure \ref{f:spindly},
whereas we show vanilla gradient descent on this network breaks the lower bound.

\subsection{Lower bound for least-squares regression}
\label{s:lower_bounds}

We demonstrate how Theorem
\ref{t:gen_lower_bound} can provide a quantitative lower
bound for a specific class of learning problems. We then show how this bound is easily beaten by non-rotation-invariant algorithms in Section \ref{s:upper_bounds}.
In the specific problem class we consider, the number of training examples is $n=md$ for some integer $m$, 
and $\X$ consists of $m$ stacked copies of a matrix
$\H=\sqrt{d}\underset{d,d}{\V}$ (i.e. $\X = [\underset{\times m}{\H; \ldots ;\H}]$),
where $\V$ is a random orthogonal matrix distributed according to the Haar measure.%
\footnote{We multiply $\V$ by $\sqrt{d}$ in order to keep the lengths of examples $\|\x_i\|$ (rows of $\X$) equal to $\sqrt{d}$, so that their coordinates $x_{ij}$ (individual features) are of order $1$ on average; any other scaling would work as well, resulting only in a relative change of the effective label noise level.}
Thus $p_{\rm in}(\X)$ is rotationally symmetric. The test input is one of the rows of $\H$,  $\Xte = \h_k$, with index $k$ drawn uniformly at random.
We assume the labels are the first feature of $\Xfull$ plus Gaussian noise, that is
\begin{equation}
q(\yfull|\Xfull) = \mathcal{N}(\yfull | \Xfull \e_1, \sigma^2 \I_{n+1}), \qquad 
    \text{ where }\e_1=(1,0,\ldots,0)^\top,
\label{eq:model}
\end{equation}
which can be be equivalently written as
$$\y = \X \e_1 + \bm{\xi}, \text{ where } \bm{\xi} \sim N(\bm{0}, \sigma^2 \I_{md}),
\;\text{ and }\;\,
\yte = \h_k^\top \e_1 + \xi_{\rm{te}}, \text{ where } \xi_{\rm{te}} \sim N(0, \sigma^2). 
$$
Note that while the inputs are shared in the training and the test parts, the test label
is generated using a `fresh' copy of the noise variable $\xi_{\rm{te}}$.
The fixed choice of $\e_1$ as opposed to any other $\e_i$ is made w.l.o.g. Indeed, Theorem \ref{t:gen_lower_bound} implies that a rotationally invariant algorithm will have the same loss for $\X \e_1$ as it will for $\X \w$ for any other unit vector $\w$.

The accuracy of prediction $\yh = \yh(\cdot | \X,\y)$ on the test set $(\h_k, \yte)$ 
is measured by the \emph{squared loss} $\mathcal{L}(\yh,\yte) = (\yh - \yte)^2$. 
Let $\predy$ denote the vector of predictions for all possible test instances, $\yh_k = \yh(\h_k | \X, \y)$, and let $\y_{\rm{te}}$ denote the vector of test labels, $y_{te_{k}} = \h_k^\top \e_1 + \xi_{\rm{te}}$, which can be
jointly written as $\yte = \H \e_1 + \xi_{\rm{te}}\bm{1}$ with $\bm{1} = (1,\ldots,1)$.
The expected value of the loss over the random choice of $k \in \{1,\ldots,d\}$ and over the independent test label noise is given by:
\begin{align*}
\EE_{k,\xi_{\rm{te}}}[\mathcal{L}(\yh,\yte)] &= \frac{1}{d} \EE_{\xi_{\rm{te}}} \left[ \|\predy - \y_{\rm{te}}\|^2 \right]
= \frac{1}{d} \EE_{\xi_{\rm{te}}} \left[ \|\predy - \H \e_1 + \xi_{\rm{te}}\bm{1} \|^2 \right] \\
&= \frac{1}{d} \|\predy - \H \e_1\|^2
+ \frac{2}{d} \underbrace{\EE_{\xi_{\rm{te}}} \left[\xi_{\rm{te}}\right]}_{=0} (\predy - \H \e_1)^\top  \bm{1} 
+ \frac{1}{d} \EE_{\bm{\xi}_{\rm{te}}} \underbrace{\left[ \xi_{\rm{te}}^2 \right]}_{=\sigma^2} \|\bm{1}\|^2 \\
&= \frac{1}{d} \|\predy - \H \e_1\|^2 + \sigma^2.
\end{align*}
Clearly, the expression above is minimized by setting the prediction vector to $\predy^{\star} = \H\e_1$, and thus the
smallest achievable expected loss is equal to $\sigma^2$.
Subtracting this loss, we get the expression for the excess risk of the learning algorithm, which we call the \emph{error} of $\predy$:
\begin{align*}
e(\predy) = \EE_{k,\xi_{\rm{te}}}[\mathcal{L}(\yh,\yte)] -
    \EE_{k,\xi_{\rm{te}}}[\mathcal{L}(\yh^{\star},\yte)] =
    \frac{1}{d} \|\predy - \H \e_1\|^2.
\end{align*}
When the prediction is \emph{linear}, $\predy = \H \predw$ for
some weight vector $\predw \in \mathbb{R}^d$, we can also refer to the error of $\predw$
as the error of its predictions:
\begin{equation}
e(\predw) = \frac{1}{d} \|\H \predw - \H \e_1\|^2
= \frac{1}{d} (\predw - \e_1)^\top \underbrace{\H^\top \H}_{d \I} (\predw - \e_1)
= \|\predw - \e_1\|^2.
\label{eq:error_w}
\end{equation}
%

We prove the following lower bound for rotationally invariant algorithms on this problem:

\begin{theorem}
Let $\underset{d,d}{\V}$ be a random orthogonal matrix, and let $\H = \sqrt{d} \V$. Let $(\X,\y)$ be the training 
set with $\underset{md,d}{\X}=[\H; \ldots ;\H]$ and labels $\y$ generated according to \eqref{eq:model}. 
Then the expected error (with respect to $\V$) of any rotation-invariant learning algorithm is at least
\[ \EE_{\V}[e(\predy)] \ge \frac{d-1}{d} \frac{\sigma^2}{\sigma^2 + m}.\]
\label{thm:lower_bound}
\end{theorem}
\vspace{-7mm}

The proof in Appendix~\ref{a:thm2-proof} uses Theorem~\ref{t:gen_lower_bound} and closely follows Section \ref{s:gen_lower_bound}: 
We start with a Bayesian setting with the rotated observation model
$q(\yfull|\Xfull\U^\top)$ for a random orthogonal matrix $\U$. 
This is equivalent to simply rotating the
target weight vector by $\U^\top$, because $\Xfull \U^\top \e_1 = \Xfull (\U^\top
\e_1)$.  Thus we can equivalently consider a linear model $\yfull = \Xfull \w
+ \tilde{\bm{\xi}}$, where $\w$ is drawn
uniformly from the unit sphere $\mathcal{S}^{d-1} = \{\w \in \mathbb{R}^d \colon
\|\w\|=1 \}$.
Given the square loss function and linear observation model, the optimal Bayes predictor is based on the posterior mean, $\mathbb{E}[\w | \X, \y]$.

Even though the posterior mean does
not have a simple analytic form for the prior distribution over a unit sphere, we use results
from \citep{Marchand,Dickel} to show that the Ridge Regression (RR) predictor (which is
the Bayes predictor for the Gaussian prior) with appropriately chosen regularization constant has expected error at most $\frac{1}{d} \frac{\sigma^2}{\sigma^2 + m}$ larger
than that of the Bayes predictor. Thus, it suffices to analyze the RR predictor
which we show to be at least
$\frac{\sigma^2}{\sigma^2 + m}$. This implies the
Bayes error is at least $\frac{d-1}{d} \frac{\sigma^2}{\sigma^2 + m}$, and no other algorithm can achieve any better error for this problem.
Finally, due to the rotation-symmetric distribution of the inputs, we can now apply Theorem  
\ref{t:gen_lower_bound} which implies that every rotation
invariant algorithm has error at least
$\frac{d-1}{d} \frac{\sigma^2}{\sigma^2 + m}$ for the original sparse linear regression
problem $\yfull = \Xfull \e_1 + \tilde{\bm{\xi}}$.
For the sake of completeness, we also give an i.i.d.
version of this lower bound in Appendix \ref{app:lower_bound_iid}.

Previous lower bound proofs for sparse problems for the
case when the number of examples is less than the input
dimension used fixed choices for input matrix $\X$ such as the $d$ dimensional Hadamard matrix \citep{percwinn,span}. 
In the over-constrained case, the lower bound does not hold for any
fixed choice of $\X$ (of full rank). 
For any fixed full-rank $\X$, there exists a row vector $\v$
s.t.~$\v \X = \e_1^\top$, and the linear algorithm $\yh(\x | \X, \y) = \v\X\, \x$ 
achieves minimal loss $\sigma^2$ while being trivially rotationally invariant
because $\v \X \U^\top \,\U\x=\v\X\,\x$.
Thus the assumption of rotationally symmetric input distribution is essential.%
\footnote{In the Hadamard matrix based lower bounds of \cite{span},
the ease of learning say $\e_1$ is overcome by averaging over all $n$ targets $\e_i$.}

This counterexample requires knowledge of the order of the
examples (i.e., the algorithm makes different predictions
if the rows of $\X$ are permuted). Now consider $\X =
[\H;\dots;\H]\, {\rm Diag}(\p)$ with $p_1=2,p_{i>1}=1$.
Then the target $\e_1$ is embedded as the first principal
component $\mathbb{PC}_1(\X^\top\X)$, so an algorithm that
used this as its weight vector, $\yh(\x|\X,\y) =
\mathbb{PC}_1(\X^\top\X)^\top\x$, would achieve minimal
loss. This algorithm is rotationally invariant but fails if
the input matrix $\X$ was rotated.

%


\section{Upper bounds}
\label{s:upper_bounds}
We now show how to break the above lower bound of
$\frac{d-1}{d} \frac{\sigma^2}{\sigma^2 + m}$ on the error
(excess risk) of rotation invariant algorithms. 
We focus on the upper bounds for EG$^\pm$ and Approximated EGU$^\pm$.
(Similar upper bounds for the spindly network, as well as
the novel priming method \citep{priming} are given in
appendices \ref{app:spindly} and \ref{app:priming},
respectively.)
We use the same setup as
in the lower bound, i.e. $\X$ consists of $m$ stacked copies of
a matrix $\H = \sqrt{d}\V$, when $\V$ is a rotation matrix
and $\y$ is sparse linear (the first component) plus Gaussian noise with variance $\sigma^2$. 
The only difference is that for the lower bound, $\V$ was
randomly chosen, but the upper bounds hold for {\em any} rotation
matrix $\V.$ The upper bound on the error that we
obtain%
\footnote{Note that at this point, we do not consider general input
matrices $\X$ in the upper bounds. Allowing arbitrary covariance
structure makes the analysis much more complicated and
the general case is left for future research.}
is essentially $O(\sigma^2 \frac{\log d}{md})$.

We begin with a bound for the normalized version of the
multiplicative update called EG$^\pm$. 
The proof relies on the facts
(a) under large learning rate, the EG$^\pm$ weight estimate becomes argmax over the negative gradient coordinates, and 
(b) with high probability, the least component of the gradient is the first one, so that the argmax is $\e_1$.
EG$^\pm$ makes use of the fact that the norm of the linear
target $\e_1$ is 1 and this additional knowledge allows the speedup.%
\footnote{We can also provide an upper bound
on the error of the online version of EG$^{\pm}$ for an arbitrary input matrix $\X$ 
with fixed feature range via a standard
worst-case regret analysis followed by the online-to-batch
conversion (see e.g. \cite{eg}). 
The bound so obtained would however give a slower rate of
order $O(\sqrt{\log d / (dm)}$, which still has a substantially better dependence on dimension 
$d$ than the lower bound from the previous
section.} We then prove our main upper bound for the unnormalized EGU$^\pm$.

\subsection{Upper bound for the Exponentiated Gradient update}

We consider a batch version of the \emph{2-sided
Exponentiated Update algorithm} (EG$^{\pm}$) \citep{eg}.
The batch EG$^{\pm}$ algorithm maintains two vectors, $\v^+_t$ and $\v^-_t$, and its weight estimate is given by $\w_t = \v^+_t - \v^-_t$. 
It starts with a set of weights $\v_1^+ = \v_1^- = \frac{1}{2d} \1$ (so that $\|\v_1^+\|_1 + \|\v_1^-\|_1 = 1$)  and updates according to
\[
\v^+_{t+1} \propto \v^+_t \odot e^{-\eta \nabla \emploss(\w_t)}, \qquad \v^-_{t+1} \propto \v^-_t \odot e^{\eta \nabla \emploss(\w_t)},
\]
where $\odot$ is component-wise multiplication, and the normalization enforces
$\|\v^+_{t+1}\|_1 + \|\v^-_{t+1}\|_1 = 1$.

\begin{theorem}
The expected error of the batch EG$^{\pm}$ algorithm after the first iteration is bounded
by
\[
e(\w_2) \le 2 d e^{-\eta} + 8 d e^{-\frac{md}{32\sigma^2}}.
\]
\label{thm:EG}
\end{theorem}
\vspace{-7mm}
\begin{proof}{\bf sketch:}
The gradient of the loss can be written as $2 (\w - \e_1 -
\bm{\zeta})$, where  $\bm{\zeta}$ are i.i.d. Gaussian noise variables which
are combinations of the original noise variables. Using the
deviation bound for Gaussians together with union bound, with probability at
least $1-2d \exp\{-md / (32 \sigma^2)\}$ all noise variables are bounded by $1/4$.
Thus, the first coordinate of the negative gradient is larger than all other
coordinates by at least $1$, so that the first weight exponentially dominates all
the other weights already after one step of the algorithm, and the error drops
to $2 d e^{-\eta}$. If the high probability event does not hold, we bound
the error by its maximum value $4$. See Appendix~\ref{a:EGpm-proof} for full proof.
\end{proof}

\subsection{Upper bound for the Approximated Unnormalized Exponentiated Gradient update}

We now drop the normalization of EG$^\pm$ and use the approximation $\exp(x)\approx 1+x$.%
\footnote{This approximated version of EGU was introduced in
\citep{eg}. It was also used in the normalized update PROD
\citep{prod}.}
The resulting approximation of the unnormalized EGU$^\pm$ algorithm updates its weights as
follows:
\begin{equation}
\w_{t+1} = \w_t - \eta\sqrt{\w_t^2 + 4\beta^2}\; \nabla \emploss(\w_t),
\label{eq:spindly_update_main}
\end{equation}
where $\beta > 0$ is a parameter and all operations
(squaring the weights and taking square root) are done
component-wise (see Appendix~\ref{a:ApproxEGU-upper} for a derivation). When $\beta=0$ this update is closely related to gradient
descent on the spindly network of Figure \ref{f:spindly}.
We start with $\w_1=\bm{0}$. Note that
unlike EG$^{\pm}$ the update does not constrain
its weights by normalizing. Nevertheless, we show that the algorithm achieves an upper bound
on the error that is essentially $O(\frac{d}{\log d})$ better than the error of
any rotation invariant algorithm:

\begin{theorem}
Assume $d \ge 4$ is such that $\sqrt{d}$ is an integer. Consider the Approximated EGU${}^{\pm}$ algorithm \eqref{eq:spindly_update_main} with $\beta=1/(2d)$ and learning rate $\eta=1/4$ and let $m \ge 8 \sigma^2 \ln \frac{2 d}{\delta} = \Omega(\sigma^2 \log (d/\delta))$ . With probability at least $1-\delta$, the algorithm run for $T=4\sqrt{d}$ steps achieves error bounded by:
\[
e(\w_{T+1}) \le \frac{10 \sigma^2 \ln \frac{2d}{\delta}}{md} + 9 e^{-\frac{8}{3} \sqrt{d} + 2 \ln d}
= O(\frac{\sigma^2\log d}{md} + e^{-\frac83\sqrt{d}})
\]
\label{thm:spindly}
\end{theorem}
\vspace{2mm}
\begin{proof}{\bf sketch:} (full proof in Appendix \ref{a:ApproxEGU-upper})
Similarly as for EG$^{\pm}$, with high probability all noise variables are small. We can interpret the weight update \eqref{eq:spindly_update_main} as the gradient descent update on the weights with the
\emph{effective learning rates} $\eta \sqrt{\w_t^2 + 4 \beta^2}$. Since $\beta=1/(2d)$
and $\w_1 = \bm{0}$, these learning rates are initially small and of order $O(1/d)$. We then show by a careful analysis of the update that the weights $w_{t,i}$ remain small for all coordinates $i$ except $i=1$, and so do the associated learning rates. Therefore, after $T$ steps, the algorithm does not move significantly away from zero on these coordinates. Meanwhile, the weight on the first coordinate $w_{t,1}$ increases towards $1$. While the initial rate of increase is small as well, it accelerates over time as $\eta \sqrt{w_{t,1}^2 + 4 \beta^2}$ increases due to increasing $w_{t,1}$. Eventually $w_{T+1,1}$ gets very close to $1$ after $T$ steps of the algorithm.
\end{proof}

\section{Gradient flow trajectories}

In our setting of overconstrained optimization without explicit regularization, all algorithms operate on the same loss landscape but follow different paths to the minimum.
To understand these differences, we derive exact expressions for the weight trajectories for the continuous time (gradient flow) versions of several representative algorithms.
We then use these results to analyze the algorithms' ability to learn sparse targets without overfitting to noise.

\subsection{Preconditioning, mirror descent, reparameterization, and Riemannian descent}

Gradient descent (GD) on a given loss function can be generalized in several ways, including preconditioned GD, mirror descent (MD), GD on reparameterized variables, and Riemannian GD. We show how they are all interchangeable in the continuous-time case.

Preconditioned GD premultiplies the gradient with a matrix $\bm{P}$ that can depend on the current weights or the training history. Explicit preconditioning methods include adaptive learning rate algorithms such as Adagrad \citep{duchi2011adaptive}, natural gradient descent which preconditions with the inverse Fisher information matrix \citep{amari1998natural}, and variational Bayesian methods that precondition with the prior or posterior covariance \citep{lambert2022recursive,chang2023low}.
\begin{equation}
\dot{\w} = -\bm{P} \nabla_w L
\end{equation}

Reparameterizing refers to defining some injective function $\widehat{\w} = g(\w)$ and performing GD on the new variable:
\begin{equation}
\dot{\widehat{\w}} = -\nabla_{\widehat{\w}} L
\end{equation}

MD \citep{MDref} involves an invertible mirror map $\widetilde{\w} = f(\w)$ where $f$ is the gradient of a convex function, with the update
\begin{equation}
\dot{\widetilde{\w}} = -\nabla_{\w} L
\end{equation}

Riemannian GD generalizes to parameter spaces with non-Euclidean geometry (specifically, differentiable manifolds) \citep{bonnabel2013stochastic}. In Riemannian geometry the update $\dot{\w}$ lies in the tangent space $T_{\w}$ (a direction of motion based at $\w$) whereas the gradient $\nabla_{\w} L$ lies in the cotangent space $T_{\w}^*$ (the dual of $T_{\w}$).
Unlike in Euclidean geometry, $T_{\w}$ and $T_{\w}^*$ cannot be identified and instead mapping between them requires a metric $\bm{\Gamma}_{\w}:T_{\w}\to T_{\w}^{*}$ or equivalently $\bm{\Gamma}_{\w}:T_{\w}\times T_{\w}\to\mathbb{R}$.
This leads to an update that is the direction of steepest descent with respect to the geometry:
\begin{equation}
\dot{\w} = -\bm{\Gamma}_{\w}^{-1} \nabla_{\w} L
\end{equation}

\begin{theorem}
\label{t:equivalence}
The continuous-time versions of reparameterized GD, MD, and Riemannian GD are all equivalent to preconditioned GD under the relation
\begin{equation}
\bm{P} 
= \frac{\partial\w}{\partial\widehat{\w}} \left(\frac{\partial\w}{\partial\widehat{\w}}\right)^\top
= \frac{\partial\w}{\partial\widetilde{\w}}
= \bm{\Gamma}_{\w}^{-1}
\end{equation}
\end{theorem}

Appendix~\ref{a:equivalence} gives the proof and then uses this result (a) to derive a close connection between EGU and the spindly network and (b) to show how EGU implicitly operates on a non-Euclidean geometry 
(Figure~\ref{fig:EGU-metric}) that tends to hold it in sparse regions of parameter space.
Rotationally invariant algorithms cannot do this because their implicit geometry must be rotationally symmetric.

\subsection{Trajectory solutions}
\label{s:trajectories}

We derive analytic trajectories for the simple linear regression problem from Sections \ref{s:lower_bounds} and \ref{s:upper_bounds}, for the continuous-time versions of EGU, EGU$\pm$, primed GD, and Adagrad \citep{adagrad}. For space reasons we present the expressions for the trajectories here and provide the derivations and full details in Appendix~\ref{a:trajectories}.

For continuous EGU the trajectory is (with $c_i$ a constant depending on initial conditions):
\begin{align}
\label{e:EGU-trajectory}
w_i(t) &= \tfrac{1}{2}w^{\rm LS}_i \left(1 + \tanh \left( w^{\rm LS}_i t + c_i \right) \right)
\end{align}

For continuous EGU$\pm$ the trajectory is (with $\tau_i$ a transform of $t$):
\begin{align}
\label{e:EGUpm-trajectory}
w_i(t) &= 
	\frac{w^{\rm LS}_i \sinh\tau_i + 1}
	{\sinh\tau_i - w^{\rm LS}_i}
\end{align}

For primed gradient flow, which is similar to the priming method analyzed elsewhere in this paper except the second stage uses gradient flow instead of ridge regression, the trajectory is
\begin{equation}
\label{e:primeGD-trajectory}
w_i(t) = w^{\rm LS}_i + \left(w_i(0) - w^{\rm LS}_i \right) e^{-2t\left(w^{\rm LS}_i\right)^2}
\end{equation}

For continuous-time Adagrad the trajectory is (with constants $k_i,\ell_i$ depending on initial conditions):
\begin{align}
\label{e:Adagrad-trajectory}
w_i(t) &= w^{\rm LS}_i - {\rm sign}(w^{\rm LS}_i - w_i(0)) \sqrt{-\frac{8}{\beta k_i} W\left(-e^{-k_i(t+\ell_i)}\right)}
\end{align}
with Lambert's W function defined by $x = W(x)e^{W(x)}$.

\subsection{Experimental visualization}

The analytic trajectories derived above are visualized in
Figure~\ref{fig:visualizations}a for a 2d version of our
noisy sparse regression problem. This toy problem is sufficient
to bring out key differences among the algorithms. 
Because all considered algorithms optimize the same loss
they all converge to the same least squares solution for
each training set (grey dots), but they follow qualitatively different trajectories. 
Priming and EGU$\pm$ pass close to the sparse target (black
dot) while GD goes straight to the LLS solution in each
run. Surprisingly Adagrad veers away from the
sparse target.%
\footnote{In Theorem \ref{thm:EG} we also prove exponential
convergence of the error of EG$^\pm$. We deemphasize this
algorithm because it normalizes based on the norm of the
true weight vector and this may be unreasonable.
Indeed using learning rate $\eta=200$, EG$^\pm$ achieves error below $10^{-300}$
(not shown). Larger learning rates cause numerical
instabilities.}

These differences in trajectory have dramatic implications for excess test loss as shown in Figure~\ref{fig:visualizations}b. 
Priming and EGU$\pm$ have deep dips in the curves
corresponding to the fact that their trajectories pass
close by the target while the dips of GD and ridge
regression are too small to be visible.
All algorithms require early stopping to
``catch'' bottom dip of their error curves.
Note that the rotation invariant algorithms both reach but do not beat the lower
bound of Theorem \ref{thm:lower_bound} while the non-rotation invariant algorithms greatly outperform the upper
bound of Theorem \ref{thm:spindly}, showing that the bound
we were able to prove is conservative. 

Finally, the bounds of Sections \ref{s:lower_bounds} and \ref{s:upper_bounds}
and the curves of Figure \ref{fig:visualizations} assume the instance matrix
$\X$ comprises $m$ copies of a
scaled $d$ dimensional rotation matrix. This unusual setup was needed for technical reasons.
However, if we used a Gaussian instance matrix
$\underset{md,d}{\X}$ together with sparse
linear targets plus Gaussian noise, then the curves 
would remain essentially unchanged (not shown).

\begin{figure}[t!]
    \centering
    \begin{center}
    \subfigure[2d-trajectories]{\includegraphics[width=0.45\linewidth]{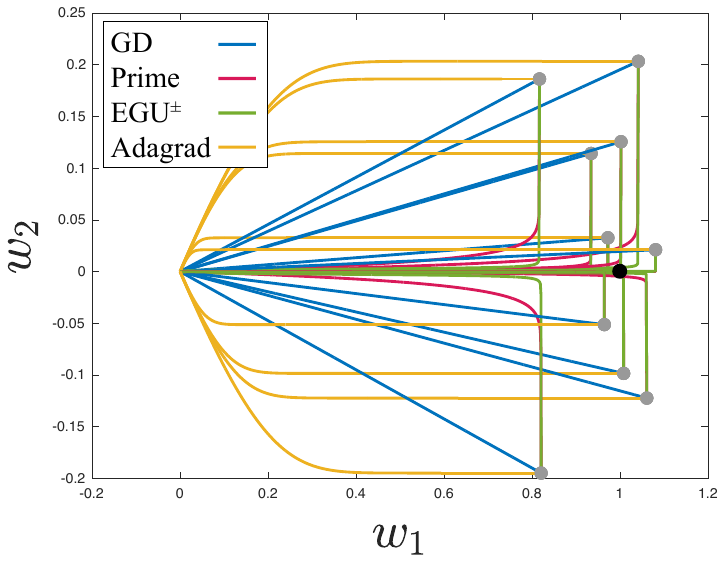}}
    \subfigure[Loss curves]{\includegraphics[width=0.53\linewidth]{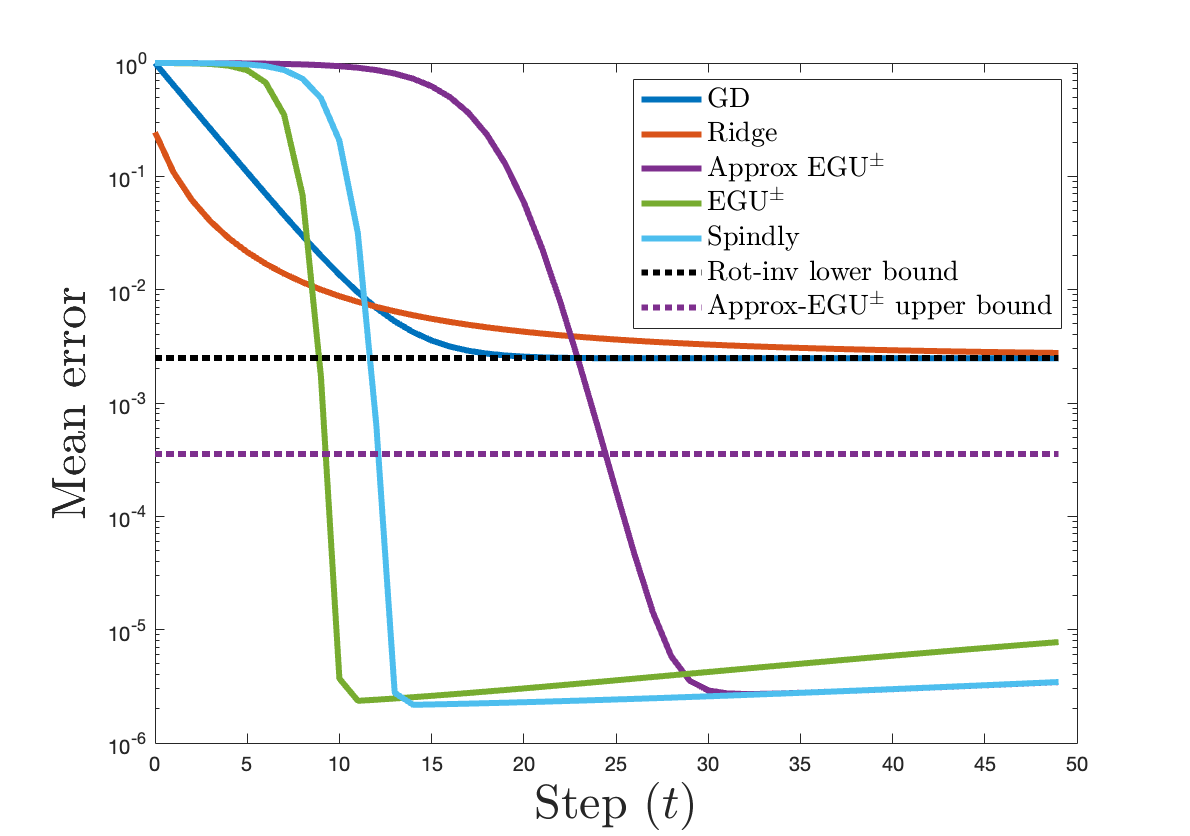}}
    \end{center}
    \caption{
    (a): Trajectories for key algorithms on 10 sampled datasets in toy 2d regression. Black dot shows true sparse target. Grey dots show least squares solution for each dataset. All algorithms converge to LS solution but differ in the order of learning the signal versus noise. EGU$\pm$ and priming learn the sparse target efficiently, overfitting to the noise only at the very end of training. Because GD is rotation invariant it cannot produce this behavior and instead learns the signal and noise at equal rates. Intriguingly, Adagrad shows the opposite behavior and is especially inefficient at learning a sparse target, because of its per-parameter adaptive learning rates.
    (b): Excess test loss averaged over 100 runs of linear regression with $d=1024$. Dotted lines show the lower bound for rotationally invariant algorithms from Theorem \ref{thm:lower_bound} (applies to GD and Ridge) and the upper bound for Approximated EGU$^{\pm}$ from Theorem \ref{thm:spindly} (using $\delta=.001$).
    }
\label{fig:visualizations}
\end{figure}

\section{Noise experiments on Fashion MNIST}

The assumption of rotationally symmetric input distribution is too strong for real datasets. Therefore it is important to complement our theoretical results with experimental demonstration that standard rotationally invariant algorithms cannot exploit asymmetries in real data. Here we report experiments on the Fashion MNIST~\citep{xiao2017fashionmnist} dataset. 
One could trivially cheat this task by hard-coding a lookup table of the test set, 
and this can even be done with a rotation invariant algorithm:
If $\X^\top\X$ is the (known, asymmetric) covariance of the original training set, 
then when presented with rotated training data
$\Z=\X\U^\top$, the algorithm infers the unknown $\U$ to
high precision by solving $\Z^\top\Z=\U\X^\top\X\U^\top$.
It then un-rotates $\bm{z}_{\rm te}$ and uses the lookup table.
Such a bespoke algorithm is of course not of interest; the question is whether standard 
rotation invariant algorithms based on gradient descent can do anything similar, or whether they are limited by their rotation invariance just as they are in the idealized symmetric-input setting.

We use a
multilayer feedforward network with two hidden layers of
size 256 each. We consider two cases for the input layer
weights: 1) fully-connected (each 1st layer hidden node is
connected to all inputs) and 2) ``spindly'' (each 1st layer hidden node
is connected to all inputs via the network of Figure
\ref{f:spindly}).
In the noise-free case, both variants of the network
achieve the same top one accuracy of $85\%$ (although
the spindly version takes longer to train).
Next, we double the number of features of
the examples by augmenting each example with uniformly sampled noise pixels in the range $[-1, 1]$. 
With noisy augmentation, the spindly network still achieves $85\%$ test accuracy 
while the fully-connected network gets only to $71\%$
(after a much longer training phase). 
In addition, the learned weights by
the two networks are significantly different: The
fully-connected network assigns almost the same magnitude
of weights to the noisy features as to the image features,
while the spindly network allocates much larger weights to the image features. 
This illustrates that rotation invariant algorithms have a harder time ignoring
the noisy features. 

Finally, to further compare the different sensitivity to the
informativeness of features, we augment each example on top of the noise
with its one-hot representation of the 10 target class labels. 
This splits the features into three categories 
in terms of their informativeness: 1) highly informative label features, 2) less informative image features, and 3) noise features with no (structured) information. The spindly network achieves $100\%$ test accuracy while the fully-connected network gets to $98\%$. We observe that the spindly network assigns much larger weights (in magnitude) to the label features while almost ignoring the rest. This phenomenon is less prominent in the fully-connected network. We defer all the details to Appendix~\ref{a:fMNIST}.

\section{Open problems}
The ``full versions'' of many common optimization
algorithms \citep{optsurvey} such as
AdaGrad, Fisher, Adam, RMS Prop are all rotation invariant.
However in practice diagonalized versions of these updates
are used. In a major step forward we were able to solve the
differential equations for all these algorithms
when $\bm{\Sigma}=\lambda\I$ (the arbitrary $\bm{\Sigma}$ cases are
challenging open problems).
This lets us analyze the inductive biases of all these
algorithms on linear neurons based on their gradient flow trajectories.
For example, the diagonalized AdaGrad is biased away from sparsity.
So using sparsity as a yard stick, we show 
that already for linear neurons, dramatic differences 
can occur compared to the optimal algorithms.
The question is whether these differences remain
when running the same algorithms on neural nets with any number of hidden layers.
\newpage
\bibliography{refs}

\appendix

\section{Proof of the lower bound Theorem \ref{t:gen_lower_bound}}
\label{a:thm1-proof}

\begin{proof}
For any orthogonal $\U$, the algorithm's expected loss is
\begin{align*}
L_{\yh}(q_{\U}) &= \mathbb{E}_{\Xfull\sim p_{\textrm{in}}, \yfull\sim q_{\U}(\cdot|\Xfull),Z}[\mathcal{L}(\yh(\Xte | \X,\y), \yte)] \\
&= \mathbb{E}_{\Xfull\sim p_{\textrm{in}}, \yfull\sim q(\cdot|\Xfull \U^\top),Z}[\mathcal{L}(\yh(\Xte | \X,\y), \yte)] \\
&= \mathbb{E}_{\Xfull'\sim p_{\textrm{in}}, \yfull\sim q(\cdot|\Xfull'),Z}[\mathcal{L}(\yh(\Xte' | \X',\y), \yte)] \\
%
&= L_{\yh}(q),
\end{align*}
where $\Xfull' = \Xfull \U^\top$ 
    (we also used $\X'=\X\U^{\top}$ and $\Xte' = \U \Xte$), 
    and where the fourth line uses rotational symmetry of $p_{\textrm
    in}$ and rotational invariance of $\yh$. Therefore, presented with the problem $\mathring{q}$, the algorithm will achieve expected loss $L_{\yh}(q)$ regardless of the value of $\U$. This implies that $L_{\yh}(q)$ cannot be less than the optimal value $L_{\mathcal{B}}(\mathring{q})$.
\end{proof}

\section{Proof of lower bound Theorem \ref{thm:lower_bound}}
\label{a:thm2-proof}
Before we give the proof observe that one can equivalently think of the $m$ stacked copies of $\H$ as just a single copy of $\H$ with the variance of label noise reduced by a factor of $m$. Indeed,
if $\y_i$ represents the label vector of size $d$ associated with the $i$-th copy of $\H$, then the training loss for any weight vector $\predw$ is effectively
$$\|\y - \X \predw\|^2 = \sum_{i=1}^m \| \y_i - \H \predw\|^2
= m \|\bar{\y} - \H \predw\|^2 + \sum_{i=1}^m \|\y_i - \bar{\y}\|^2,$$
where $\bar{\y} = m^{-1} \sum_i \y_i \sim \mathrm{N}(\H \e_1, \sigma^2/m \I_d)$ and the last term on the right-hand side does not depend on $\predw$. 
Therefore, by varying $m$, we effectively alter the level of noise the algorithms encounter.
We also note that our model is equivalent to the Gaussian Sequence Model (see, for example, \citet{rigollet2023highdimensional}),
given by $\y = \bm{\theta} + \bm{\xi}$ (with $\bm{\theta} = \X \e_1$ and $\bm{\xi} \sim \mathcal{N}(\bm{0},\sigma^2/m \I_d)$), a commonly used model for analyzing nonparametric and high dimensional statistical problems.

\vspace{1mm}
\begin{proof}
We consider the rotated observational model described in the main text, which can be defined as follows. Let $\w \in \mathbb{R}^d$
be a weight vector drawn uniformly from a unit sphere 
$\mathcal{S}^{d-1} = \{\w \in \mathbb{R}^d \colon \|\w\|=1 \}$.
The algorithm is given data set $(\X,\y)$ with $\X = [\H ; \ldots ; \H]$ being
$m$ copies of $\H=\sqrt{d} \V$, and $\y = \H \w + \bm{\xi}$, where $\bm{\xi} \sim N(\bm{0},
\sigma^2 \I_{dm})$ is a vector of Gaussian i.i.d. noise variables, each having zero
mean and variance $\sigma^2$. Given $\y$, the algorithm is supposed to produce a vector of predictions
$\predy\in \mathbb{R}^d$ and is evaluated by means of the squared error,
$e(\predy | \w) = \frac{1}{d} \|\predy - \H \w\|^2$.
We first note that without loss of generality, the algorithm produces a weight vector
$\predw$, based on which the predictions are generated, $\predy = \H \predw$; this is due
to the fact that $\H$ is invertible (as multiplicity of an orthogonal matrix), so
for every $\predy$, one can have a corresponding weight vector $\predw = \H^{-1} \predy$.
Thus, using \eqref{eq:error_w} the error can be equivalently written as
\[
e(\predw | \w) = \|\predw - \w\|^2.
\]
It is well-known (see, e.g., \citet{Berger}) that the expected squared error,
$\EE_{\w,\bm{\xi}}\left[e(\predw | \w)\right]$ (with expectation with
respect to the prior and the label noise) is minimized by the \emph{posterior mean}
$\predw^{\star} = \EE_{\w|\y} [ \w]$, that is the mean value of $\w$ with respect to the posterior distribution 
 $q(\w|\X,\y)$. 
 Even though the posterior mean does not have a nice analytic form, we can still lower bound
 its expected squared error using a technique borrowed from \citep{Marchand,Dickel}. Let
 $\predw_{RR}$ be the ridge regression estimator:
 \[
 \predw_{RR} = (\X^\top \X + \sigma^2 d \I)^{-1} \X^\top \y,
 \]
 which is the posterior mean itself (and thus optimal)
 when the prior over $\w$ is Gaussian with zero mean and covariance $\frac{1}{d} \I$ (the covariance is multiplied by factor $d^{-1}$ to have $\EE[\|\w\|^2] = \EE[\w \w^\top] = \tr(\frac{1}{d}\I) = 1$ as in the unit sphere prior case). Even though the Gaussian prior differs from the uniform prior over a unit sphere, it turns out that the RR predictor has error only slightly larger than that of the optimal Bayes predictor 
 $\predw^{\star}$:
 \begin{lemma}
 \[
\EE_{\w,\bm{\xi}}\left[e(\predw_{RR} | \w)\right] \le
 \EE_{\w,\bm{\xi}}\left[e(\predw^{\star} | \w)\right]
 + \frac{1}{d} \frac{\sigma^2}{\sigma^2 + m}.
 \]
 \label{lem:RR}
 \end{lemma}
 \begin{proof}
 \citet{Dickel} considered a Bayesian setting similar to ours, with $\sigma^2 = 1$ and
 $\w$ distributed uniformly over $\tau \mathcal{S}^{d-1} = \{\w \colon \|\w\|=\tau\}$.
 To account for this setting, we note that in our setup,
 \[
\sigma^{-1} \y = \X^\top (\sigma^{-1} \w) ~+ \underbrace{\sigma^{-1} \bm{\xi}}_{\sim N(\bm{0}, \I_{dm})},
 \]
 so that we can set $\tau = \sigma^{-1}$ and assume unit variance of the noise. Their `oracle ridge estimator' is thus $\predw_{RR}$. Furthermore, since $\|\predw - \w\|^2 =
 \sigma^2 \|\sigma^{-1} \predw - \sigma^{-1} \w\|^2$, we need to multiply their bound by $\sigma^2$. We use their Theorem 2 (adapted to the modifications stated above):
 \begin{theorem}{(Theorem 2 by \citet{Dickel}, Theorem 3.1 by \citet{Marchand})}
 Let $n = md$ and let $s_1 \ge \ldots \ge s_d$ denote the eigenvalues of $n^{-1} \X^\top \X$.
 Then
 \[
\EE_{\w,\bm{\xi}}\left[e(\predw_{RR} | \w)\right] 
\le \EE_{\w,\bm{\xi}}\left[e(\predw^{\star} | \w)\right] 
+ \frac{\sigma^2}{d} \frac{s_1}{s_d} \mathrm{tr} \left\{ (\X^\top \X + d \sigma^2 \I_n)^{-1}\right\}
. \]
 \end{theorem}
 Since $\X^\top \X = n \I$, we have $s_1 = s_d = 1$ and
 $(\X^\top \X +  \I_n)^{-1} = (n + d \sigma^2)^{-1} \I_n$, and thus,
 \[
\EE_{\w,\bm{\xi}}\left[e(\predw_{RR} | \w)\right] 
\le \EE_{\w,\bm{\xi}}\left[e(\predw^{\star} | \w)\right] 
+ \frac{\sigma^2}{n + d \sigma^2}
=\EE_{\w,\bm{\xi}}\left[e(\predw^{\star} | \w)\right] 
+ \frac{1}{d} \frac{\sigma^2}{m + \sigma^2}.
 \]

     \vspace{-9mm}
 \end{proof}
 \vspace{1mm}
Now, we compute the expected error of $\predw_{RR}$. Since $\X^\top \X = md \I$, we get
\[
\predw_{RR} = \frac{1}{md + \sigma^2 d} \X^\top \y
= \frac{1}{md + \sigma^2 d} \X^\top (\X \w + \bm{\xi}) 
= \frac{md \w + \X^\top \bm{\xi}}{md + \sigma^2 d},
\]
and thus 
\begin{align*}
e(\predw_{RR}|\w) &= \|\predw_{RR} - \w\|^2
= \left| \frac{\X^\top \bm{\xi} - \sigma^2 d \w}{md + \sigma^2 d} \right\|^2 
    \\ &= \frac{\| \X^\top \bm{\xi}\|^2}{(md + \sigma^2 d)^2} - \frac{md \w{}^\top \X^\top \bm{\xi}}{(md + \sigma^2 d)^2} + \frac{\sigma^4 d^2 \|\w\|^2}{(md + \sigma^2 d)^2}.
\end{align*}
We take the expectation over $\w$, under which the middle term in the last line
vanishes (as $\EE[\w] = \bm{0}$ over a unit sphere) and use $\|\w\|=1$ to get
\[
\EE_{\w} [e(\predw_{RR}|\w)] = 
\frac{\| \X^\top \bm{\xi}\|^2}{(md + \sigma^2 d)^2}
+ \frac{\sigma^4 d^2}{(md + \sigma^2 d)^2}.
\]
We further take an expectation over $\bm{\xi}$ and use
\[
\EE[\|\X^\top \bm{\xi}\|^2] = \EE[\tr(\X^\top \bm{\xi} \bm{\xi}^\top \X)]
= \tr(\X^\top \EE[\bm{\xi} \bm{\xi}^\top] \X) = \tr(\X^\top \X) = md \tr(\I) = md^2,
\]
to get:
\[
\EE_{\w, \bm{\xi}} [e(\predw_{RR}|\w)] = 
\frac{m d^2}{(md + \sigma^2 d)^2}
+ \frac{\sigma^4 d^2}{(md + \sigma^2 d)^2}.
= \frac{\sigma^2 d (md + \sigma^2 d)}{(md + \sigma^2 d)^2}
= \frac{\sigma^2}{m + \sigma^2}.
\]
Using this together with Lemma \ref{lem:RR} gives the lower bound on the Bayes optimal
predictor in the rotated observational model.
\[
\EE_{\w, \bm{\xi}} [e(\predw^{\star}|\w)] 
\ge \frac{d-1}{d} \frac{\sigma^2}{\sigma^2 + m}.
\]
We now use the fact that
that $\H = \sqrt{d} \V$ 
with orthogonal matrix $\V$ of size $d \times d$, drawn uniformly at random (with respect to Haar measure), so that our input distribution is rotation symmetric. This means that we can apply Theorem 1 and conclude that any rotation invariant algorithm has the expected error
at least $\frac{d-1}{d} \frac{\sigma^2}{\sigma^2 + m}$ on the original problem,
that is for $\w=\e_1$.
\end{proof}
Note that the proof would significantly simplify if we
assumed from the start that the target weight vector
$\w$ is generated from a Gaussian distribution
$N(\bm{0}, \frac{1}{d} \I_d)$ rather than from a unit
sphere $\mathcal{S}^{d-1}$ (both priors give unit squared
norm of $\w$ on expectation), as the Bayes
predictor would be exactly the RR predictor, giving even a
better lower bound of $\frac{\sigma^2}{\sigma^2 + m}$,
without the need to apply results from
\citet{Marchand,Dickel}. This would, however, result in a
random norm of the sparse target vector. In our proof we opted for a bound with a fixed, unit norm of $\w$.

\section{Lower bound for i.i.d. sampling}
\label{app:lower_bound_iid}
Instead of observing copies of the same matrix $\H$, one can instead consider a standard linear model framework, where individual inputs 
$\x$ are drawn i.i.d. from some input distribution $p_{\mathrm{in}}(\x)$. The associated labels are given by $y = \x_i^\top \e_1 + \xi$, where $\xi \sim \mathcal{N}(0,\sigma^2)$.
The algorithm observes $n$ such i.i.d. samples $(\X,\y)$ and produces a weight vector $\predw$, which is then evaluated by the squared error $\|\predw - \e_1\|^2$. Assume $p_{\mathrm{in}}(\x)$ is rotationally symmetric with covariance $\EE[\x\x^\top] = \I_d$ (so that the expected input size matches that from our previous setup); for instance, we could have $\x \sim \mathcal{N}(\bm{0}, \I_d)$. Employing our Bayesian argument from section \ref{s:gen_lower_bound},
one can establish a similar lower bound for any rotation-invariant algorithm.
Note that for large $d$, the bound is essentially the same
as that of Theorem \ref{thm:lower_bound}.  
The proof makes use of information-theoretic arguments by \citet{Gerchinovitz_etal2020}.
\begin{theorem}
Let the rows of $\X$ be drawn i.i.d. from rotationally symmetric $p_{\mathrm{in}}$ with covariance $\I_d$, and the labels given by $\y = \X \e_1 + \bm{\xi}$,
$\bm{\xi} \sim \mathcal{N}(\bm{0}, \sigma^2 \I_n)$. Assume $d \ge 12$.
Then, the expected error (with a random choice of the sample) of any rotation-invariant learning algorithm is at least
\[
\EE[e(\predw)]  \ge \left(\frac{\sigma^2}{e (\sigma^2 + m / 2)}\right)^{d/(d-1)}.
\]
\label{thm:lower_bound_iid}
\end{theorem}
\begin{proof}
We consider a Bayesian setup, in which the target weight vector is drawn
uniformly from a unit sphere,
 $\w \sim \mathcal{S}^{d-1}$,
with $\mathcal{S}^{d-1} = \{\w \in \mathbb{R}^d \colon \|\w\|=1\}$. 
We will show a lower bound on the error of \emph{any} predictor $\predw$,
$\EE_{\w,\bm{\xi}}\left[\|\predw - \w\|^2\right]$,
by exploiting the fact
that it is not less than the error of the Bayes optimal predictor $\predw^{\star}$ (the posterior mean), followed
by a lower bound on the error of $\predw^{\star}$:
\begin{equation}
\EE_{\w,\bm{\xi}}\left[\|\predw^{\star} - \w\|^2\right] \ge \left(\frac{\sigma^2}{e (\sigma^2 + m / 2)}\right)^{d/(d-1)}.
\label{eq:lower_bound_to_prove_iid}
\end{equation}
Using the rotational symmetry of the input distribution $p_{\mathrm{in}}$, and giving the same arguments as in the proof of Theorem \ref{thm:lower_bound} (or Section \ref{s:gen_lower_bound}), this will imply
the same lower bound \eqref{eq:lower_bound_to_prove_iid} for any rotation-invariant algorithm with a sparse target vector
$\w = \e_1$.


Unfortunately, for prior uniform over a sphere, $\predw^{\star}$ does not have a closed form. We will, however, use a continuous version of Fano's inequality \citep{Gerchinovitz_etal2020} to lower bound its Bayes risk. 

From now on, we condition our analysis on $\X$ drawn i.i.d. from the input distribution. 
We lower-bound the Bayes risk using Markov's inequality: for any $\epsilon > 0$,
\begin{equation}
\EE[\|\predw^{\star} - \w\|^2] \ge \epsilon P(\|\predw^{\star} - \w\|^2 \ge \epsilon)
= \epsilon \left(1 - P(\|\predw^{\star} - \w\|^2 < \epsilon) \right).
\label{eq:bound_sphere}
\end{equation}
where the expectation and the probability is with respect to random $\y, \w$ (conditioned on $\X$). Let $P_{\w} = \mathcal{N}(\X \w, \sigma^2 \I_n)$ be the distribution of $\y$ given $\w$ (and $\X$). Given $\w$, define event 
$\mathcal{A}_{\w} = \{ \|\predw^{\star} - \w\|^2 < \epsilon\}$. We can
write
\[
P\left(\|\predw^{\star} - \w\|^2 < \epsilon \right) = \EE_{\w} \left[P_{\w}(\mathcal{A}_{\w})\right]
\]
Using techniques from \citet{Gerchinovitz_etal2020},
we have for any distribution $Q$ over $\y$,
\begin{equation}
\EE_{\w} \left[P_{\w}(\mathcal{A}_{\w})\right]
\le \EE_{\w}\left[Q(\mathcal{A}_{\w})\right]
+ \sqrt{\left(\EE_{\w}\left[Q(\mathcal{A}_{\w})\right]\right)
\left(\EE_{\w} \left[\chi^2(P_{\w}, Q) \right] \right)}
\label{eq:Gerchinovitz_etal2020}
\end{equation}
where $\chi^2(P,Q) = \EE_{Q}\left[\left(\frac{\mathrm{d}P}{\mathrm{d}Q}\right)^2\right]$ is the $\chi^2$-divergence between $P$ and $Q$. 
We now choose $Q = \mathcal{N}(\bm{0}, \SSigma)$
with $\SSigma = \sigma^2 \I_n + \frac{1}{d} \X\X^\top$. 
It follow from the Gaussian integration that
\[
\chi^2(P_{\w},Q) =
\sqrt{\det \left(\sigma^{-2} \SSigma^2 (2 \SSigma - \sigma^2 \I_n)^{-1}\right)}
e^{\w^\top \X^\top \left(2 \SSigma - \sigma^2 \I_n\right)^{-1} \X \w} - 1.
\]
First note that
\[
\sigma^{-2} \SSigma^2 (2 \SSigma - \sigma^2 \I_n)^{-1}
= \left(\I_n + \frac{1}{d \sigma^2} \X\X^\top\right)^2 \left(\I_n + \frac{2}{d \sigma^2} \X \X^\top\right)^{-1} \preceq \I_n + \frac{1}{2 d \sigma^2} \X \X^\top.
\]
Indeed, by taking any eigenvalue $\lambda_i$ of $\frac{1}{2d \sigma^2} \X \X^\top$,
\[
\frac{(1 + \lambda_i)^2}{1 + 2 \lambda_i} \le
1 + \frac{\lambda_i^2}{1 + 2 \lambda_i} \le 1 + \frac{\lambda_i}{2}.
\]
Moreover, from the property of the determinant, that
$\det(\I_m + \bm{A} \bm{B}) = \det(\I_n + \bm{B} \bm{A})$ for $\bm{A}$ of size $m \times n$ and $\bm{B}$ of size $n \times m$, we get
\[
\det\left(\I_n + \frac{1}{2 d \sigma^2} \X \X^\top\right) = \det\left(\I_m + \frac{1}{2 d \sigma^2} \X^\top \X\right).
\]
Finally, note that
\[
\X^\top (2 \SSigma - \sigma^2 \I_n)^{-1} \X
= \X^\top \left(\sigma^2 \I_n + \frac{2}{d} \X \X^\top\right)^{-1} \X
\preceq \frac{d}{2} \X^\top(\X \X^\top)^{\dagger} \X = \frac{d}{2} \I_d,
\]
so that
\[
e^{\w^\top \X^\top \left(2 \SSigma - \sigma^2 \I_n\right)^{-1} \X \w} 
\le e^{\frac{d}{2} \|\w\|^2} = e^{\frac{d}{2}}.
\]
Thus, we get:
\[
\chi^2(P_{\w},Q) \le
\sqrt{\det \left(\I_m + \frac{1}{2 d \sigma^2} \X^\top \X\right)} e^{\frac{d}{2}}
\]
For any positive-definite $d \times d$ matrix $\bm{A}$ with eigenvalues $a_1,\ldots,a_d$,
we have $\ln \det \bm{A} = \sum_i \ln a_i$. Using the concavity of negative logarithm,
we can bound:
\[
\ln \det \bm{A} = d \frac{1}{d} \sum_i \ln(a_i)
\le d \ln\left(\frac{1}{d} \sum_i a_i \right)
= d \ln \tr(\bm{A} / d),
\]
which gives $\det \bm{A} \le \tr(\bm{A} / d)^d$ and thus
\[
\chi^2(P_{\w},Q) \le
\left(e \tr\left(\frac{1}{d}\I_m + \frac{1}{2 d^2 \sigma^2} \X^\top \X\right)\right)^{\frac{d}{2}}.
\]
Since this expression does not depend on $\w$, taking expectation on both sides (with respect to $\w$) gives the same bound on $\EE_{\w}\left[\chi^2(P_{\w},Q)\right]$.

We will now upper-bound $\EE_{\w}[Q(\mathcal{A}_{\w})]$. 
Let
$\bm{1}\{C\}$ denote the Iverson bracket which is $1$ if $C$ holds and $0$ otherwise.
Using the fact that $\predw^{\star} = \predw^{\star}(\y)$ depends on $\y$ (and $\X$), but
not on $\w$, and that $\y \sim Q$ is independent of $\w$, we have 
\begin{align*}
\EE_{\w}[Q(\mathcal{A}_{\w})]
&= \EE_{\w}[Q(\|\predw^{\star} - \w\|^2 < \epsilon)]
= \EE_{\w \sim \mathcal{S}^{d-1},\y \sim Q}\left[\bm{1}\left\{\|\predw^{\star}(\y) - \w\|^2 < \epsilon\right\}\right] \\
&= \EE_{\y \sim Q} \left[\EE_{\w \sim \mathcal{S}^{d-1}}\left[\bm{1}\left\{\|\predw^{\star}(\y) - \w\|^2 < \epsilon\right\}\right]\right] \\
&\le \EE_{\y \sim Q} \left[\max_{\predw} \EE_{\w \sim \mathcal{S}^{d-1}}\left[\bm{1}\left\{\|\predw - \w\|^2 < \epsilon\right\}\right]\right]
= \max_{\predw} \EE_{\w}\left[\bm{1}\left\{\|\predw - \w\|^2 < \epsilon\right\}\right],
\end{align*}
where in the last equality we used the fact that the term inside the expectation over $\y$
does not depend on $\y$. Let us now fix $\predw$ and we will bound
$\EE_{\w}\left[\bm{1}\left\{\|\predw - \w\|^2 < \epsilon\right\}\right]$
Since the distribution
of $\w$ is rotation-invariant (uniform over a sphere), without loss of generality take
$\predw = c \e_1$ with $c > 0$ ($c=0$ would give $0$ for any $\epsilon < 1$). We have
\[
\|\predw^{\star} - \w\|^2 = (c-w_1)^2 + \sum_{j=2}^d w_j^2 = 
(c-w_1)^2 + 1 - w_1^2 = 1 + c^2 - 2 c w_1, 
\]
and thus
\[
\EE_{\w}\left[\bm{1}\left\{\|\predw - \w\|^2 < \epsilon\right\}\right]
= \EE_{\w} \left[ \bm{1}\left\{1 + c^2 - 2 c w_1 < \epsilon\right\} \right]
= P\left( w_1 > \frac{1+ c^2 - \epsilon}{2c} \right).
\]
Since the probability is nonincreasing in $\frac{1+c^2-\epsilon}{2c}$ we take $c$
which minimizes this quantity to get an upper bound. By inspecting the derivative
\[
\min_c \frac{1+c^2-\epsilon}{2c}
= \frac{1}{2} \min_c c + \frac{1-\epsilon}{c} 
~\stackrel{c=\sqrt{1-\epsilon}}{=}~ \sqrt{1-\epsilon},
\]
we have
\[
\EE_{\w}\left[\bm{1}\left\{\|\predw - \w\|^2 < \epsilon\right\}\right]
= P\left(w_1 > \sqrt{1-\epsilon}\right) = \frac{1}{2} P\left(w_1^2 > 1 - \epsilon\right),
\]
where we used the fact that distribution of $w_1$ is symmetric around zero.
To determine the distribution of $w_1^2$ we use the fact that drawing $\w$ uniformly
over a unit sphere amounts to drawing $\u \sim \mathcal{N}(\bm{0}, \I_d)$ and setting
$\w = \frac{\u}{\|\u\|}$. Thus, $w_1^2 = \frac{x}{x + y}$, where
$x \sim \chi^2(1) = \mathrm{Gamma}(\alpha = 1/2, \beta = 1/2)$ and
$y \sim \chi^2(d-1) =  \mathrm{Gamma}(\alpha = (d-1)/2, \beta = 1/2)$ and $x,y$ independent,
Thus,
$w_1^2$ is distributed according to beta distribution $B(\alpha = 1/2, \beta = (d-1)/2)$.
This gives
\begin{align*}
P\left(w_1^2 > 1 - \epsilon\right)
&= \frac{\Gamma(d/2)}{\Gamma(1/2) \Gamma((d-1)/2)} \int_{1-\epsilon}^1 x^{-1/2} (1-x)^{-(d-1)/2}
\mathrm{d} x \\
&\le \frac{\Gamma(d/2)}{\Gamma(1/2) \Gamma((d-1)/2)} (1-\epsilon)^{-1/2} \int_{1-\epsilon}^1 (1-x)^{(d-1)/2-1} \mathrm{d} x \\
&= \frac{\Gamma(d/2)}{\Gamma(1/2) \Gamma((d-1)/2)} \frac{2(1-\epsilon)^{-1/2} \epsilon^{(d-1)/2}}{d-1}
\end{align*}
Assuming $\epsilon < 1/2$, we have $2(1-\epsilon)^{-1/2} \le 2\sqrt{2}$. We also use 
Gautschi's inequality which states that for any $x > 0$ and any $s \in (0,1)$,
$\frac{\Gamma(x+1)}{\Gamma(x+s)} \le (x+1)^{1-s}$. Taking
$x = \frac{d}{2}-1$ and $s=\frac{1}{2}$, we can bound
$\frac{\Gamma(d/2)}{\Gamma((d-1)/2)} \le \sqrt{d/2} \le \sqrt{d-1}$, where we used $d \ge 2$. This, together
with $\Gamma(1/2) = \sqrt{\pi}$ allows us to bound 
\[
P\left(w_1^2 > 1 - \epsilon\right)
\le \frac{2 \sqrt{2} \epsilon^{(d-1)/2}}{\sqrt{\pi} \sqrt{d-1}}.
\]
Thus, $\EE_{\w}[Q(\mathcal{A}_{\w})]$ can be bounded by
\[
\EE_{\w}[Q(\mathcal{A}_{\w})] = \frac{1}{2} P\left(w_1^2 > 1 - \epsilon\right)
\le \frac{\sqrt{2} \epsilon^{(d-1)/2}}{\sqrt{\pi} \sqrt{d-1}}
\]
Plugging all bounds to \eqref{eq:Gerchinovitz_etal2020}
gives:
\[
\EE_{\w}\left[P_{\w}(\mathcal{A}_{\w})\right]
\le \frac{\sqrt{2} \epsilon^{(d-1)/2}}{\sqrt{\pi(d-1)}}
+ \sqrt{\sqrt{2} \frac{\epsilon^{(d-1)/2}}{\sqrt{(\pi(d-1)}}
\left(e \tr\left(\frac{1}{d}\I_m + \frac{1}{2 d^2 \sigma^2} \X^\top \X\right)\right)^{\frac{d}{2}}}.
\]
Now, take
\[
\epsilon^{-1} = \left(e 
\tr\left(\frac{1}{d}\I_m + \frac{1}{2 d^2 \sigma^2} \X^\top \X\right)\right)^{d/(d-1)}.
\]
Note that $\epsilon \le 1/e$.
This gives
\[
\EE_{\w}\left[P_{\w}(\mathcal{A}_{\w})\right]
\le \frac{\sqrt{2} e^{-(d-1)/2}}{\sqrt{\pi(d-1)}}
+ \sqrt{\frac{\sqrt{2}}{\sqrt{\pi(d-1)}}} \le \frac{1}{2},
\]
whenever $d \ge 12$. We thus have shown
\[
P\left(\|\predw^{\star} - \w\|^2 \le \epsilon\right)
\le \frac{1}{2},
\]
which, using \eqref{eq:bound_sphere}, gives:
\[
\EE_{\w,\y|\X}[\|\predw^{\star} - \w\|^2] \ge \frac{\epsilon}{2}.
\]
To get a bound with respect to a random choice of $\X$, note that
function $x \mapsto x^{-d/(d-1)}$ is convex, and thus:
\begin{align*}
\EE[\|\predw^{\star} - \w\|^2] &\ge \frac{1}{2} \EE[\epsilon] 
\ge \frac{1}{2} 
\left(e 
\tr\left(\frac{1}{d}\I_m + \frac{1}{2 d^2 \sigma^2} \EE[\X^\top \X]\right)\right)^{d/(d-1)} \\
&= \left(e \left(1 + \frac{n \tau^2}{2 d \sigma^2}\right)\right)^{-d/(d-1)}
= \left(\frac{\sigma^2}{e (\sigma^2 + n \tau^2 / (2d))}\right)^{d/(d-1)}
\end{align*}
\end{proof}

\section{Upper bound for EG$^\pm$: Proof of Theorem \ref{thm:EG}}
\label{a:EGpm-proof}

Here we prove that the 
batch version of EG$^{\pm}$ achieves small error already after one trial, when the learning rate is set to a sufficiently large value.

The (batch) EG$^{\pm}$ algorithm keeps track of two vectors, $\v^+_t$ and $\v^-_t$, and its prediction vector is given by $\w_t = \v^+_t - \v^-_t$. 
It starts with a set of weights $\v_1^+ = \v_1^- = \frac{1}{2d} \1$,  and updates according to
\[
\v^+_{t+1} \propto \v^+_t \odot e^{-\eta \nabla \emploss(\w_t)}, \qquad \v^-_{t+1} \propto \v^-_t \odot e^{\eta \nabla \emploss(\w_t)},
\]
where $\odot$ is component-wise multiplication, and the normalization ensures that
$\|\v^+_{t+1}\|_1 + \|\v^-_{t+1}\|_1 = 1$, while $\emploss(\w)$ is the average total loss on the training
sample:
\[
\emploss(\w) = \frac{1}{dm} \|\X \w - \y_t\|^2.
\]
We compute the expression for the gradient:
\begin{align*}
\nabla \emploss(\w) &= \frac{2}{dm} \sum_{t=1}^m \sqrt{d} \V^\top (\sqrt{d} \V \w - \y_t)
= \frac{2}{dm} \sum_{t=1}^m \sqrt{d} \V^\top (\sqrt{d} \V (\w - \e_1) - \bm{\xi}_t) \\
&= \frac{2}{dm} \sum_{t=1}^m \left( d(\w - \e_1) + \sqrt{d} \V^\top \bm{\xi}_t \right)
= 2 (\w - \e_1) - \frac{2}{\sqrt{d}} \V^\top \avgxi,
\end{align*}
where
\[
\avgxi = \frac{1}{m} \sum_{t=1}^m \bm{\xi}_t ~\sim~ N\left(\0, \frac{\sigma^2}{m}\I\right),
\]
and the reduction of variance is due to averaging i.i.d. noise variables. Furthermore,
we rewrite
\begin{equation}
\nabla \emploss(\w) = 2(\w - \e_1 - \bm{\zeta}),
\label{eq:gradient_simplified}
\end{equation}
where the noise vector $\bm{\zeta} =  \frac{1}{\sqrt{d}} \V^\top \avgxi$ has distribution
\[
\bm{\zeta} ~\sim~ N \left(\0, \frac{\sigma^2}{m d} \V^\top \V\right)
~=~ N \left(\0, \frac{\sigma^2}{m d} \I\right).
\]
We also bound the error of the algorithm from above:
\[
e(\w_t) = \|\w_t - \e_1\|^2 = \|\w_t\|^2 - 2 \w_t^\top \e_1 + \|\e_1\|^2 \le 2 - 2 \w_t^\top \e_1
= 2(1-w_{t,1}) = 2(1-v_{t,1}^+ +  v_{t,1}^-).
\]
So it suffices to upper-bound $1-v_{t,1}^+$ and $v_{t,1}^-$.

Consider the weights of the batch EG$^{\pm}$ algorithm after just a \emph{single} trial,
that is $\v_2^+$ and $\v_2^-$
Using \eqref{eq:gradient_simplified} and noting that $\w_1 = \0$, and
$v_{1,i}^+ = v_{1,i}^- = \frac{1}{2d}$ for all $i$, we can concisely write $v_{2,1}^+$ and $v_{2,1}^-$ as:
\begin{equation}
v_{2,1}^+ = \frac{e^{2\eta\left(1 + \zeta_1\right)}}{Z_2},
\qquad
v_{2,1}^- = \frac{e^{-2\eta\left(1 + \zeta_1\right)}}{Z_2},
\qquad
Z_2 = \sum_{i=1}^d e^{2\eta\left(\delta_{1i} + \zeta_i\right)}
+ e^{-2\eta\left(\delta_{1i} + \zeta_i\right)},
\label{eq:v_plus_v_minus}
\end{equation}
We will now lower-bound $v_{2,1}^+$, and later use a relation which directly follows from
\eqref{eq:v_plus_v_minus}:
\begin{equation}
v_{2,1}^- = e^{-4\eta(1+\zeta_1)} v_{2,1}^+.
\label{eq:relation_v21}
\end{equation}

Using the deviation bound for zero-mean Gaussian
$z \sim N(0,\tau^2)$,
$P(|z| \ge \gamma) \le 2\exp\left\{-\frac{\gamma^2}{2\tau^2}\right\}$, we get 
that for any $i=1,\ldots,d$, 
\[P(|\zeta_i| \ge 1/4) \le 2 \exp \left\{- \frac{md \gamma^2}{32 \sigma^2} \right\}.
\]
Taking the union bound over $i=1,\ldots,i$, we have
\[
P(\exists i \; |\zeta_i| \ge 1/4) 
\le 2d \exp \left\{- \frac{md \gamma^2}{32 \sigma^2} \right\}.
\] Denoting the probability on the right-hand side by $\delta$, we conclude that with probability at least $1-\delta$, all noise variables $\zeta_i$ are bounded by $1/4$. Let us call this event $E$, and we condition everything that follows on the fact that 
$E$ happened. 

Note that for any $i \ge 2$,
\[
\frac{\partial v_{2,1}^+}{\partial \zeta_i}
= - \frac{v_{2,1}^+}{Z_2} 2 \eta \left(e^{2\eta \zeta_i} - e^{-2 \eta \zeta_i} \right),
\]
which is decreasing for $\zeta_i > 0$ and increasing for $\zeta_i < 0$. So,
to lower-bound $v_{2,1}^+$, conditioning on event $E$, we set $\zeta_i = 1/4$ for all $i \ge 2$ in \eqref{eq:v_plus_v_minus} ($\zeta_i=-1/4$ would result in the same value of these weights). This gives: 
\begin{align}
v_{2,1}^+ &\ge \frac{e^{2\eta(1+\zeta_1)}}{e^{2\eta(1+\zeta_1)} + e^{-2\eta(1+\zeta_1)}
+ (d-1) (e^{\eta/2} + e^{-\eta / 2})} \nonumber \\
&= \frac{1}{1 + e^{-4\eta(1+\zeta_1)} + e^{-2\eta(1+\zeta_1)}(d-1)(e^{\eta/2} + e^{-\eta / 2})} \nonumber \\
&\ge \frac{1}{1 + e^{-4\eta(1-1/4)} +  e^{-2\eta(1-1/4)}(d-1)(e^{\eta/2} + e^{-\eta / 2})} 
\nonumber \\
&\ge \frac{1}{1 + e^{-3\eta} + e^{-3/2 \eta} 2 (d-1)e^{\eta/2}} 
\nonumber \\
&\ge \frac{1}{1 + e^{-3\eta} + 2(d-1) e^{-\eta}} 
\ge \frac{1}{1+(2d-1) e^{-\eta}}
\label{eq:lower_bound_v21_plus}
\end{align}
This gives:
\[
1 - v_{2,1}^+ \le \frac{(2d-1) e^{-\eta}}{1 + (2d-1) e^{-\eta}}
= \frac{1}{1 + e^{\eta}/ (2d-1)} \le (2d-1) e^{-\eta}.
\]
To upper-bound $v_{2,1}^-$, we use \eqref{eq:relation_v21}. Conditioning on $E$,
\[
v_{2,1}^- =e^{-4\eta(1+\zeta_1)} v_{2,1}^+ \le e^{-4 \eta(1+\zeta_1)}
\le e^{-2 \eta}. 
\]
Thus, with probability at least $1-\delta$, the error can be bounded by:
\[
e(\w_2) \le 2(1-v_{2,1}^+ + v_{2,1}^-)
\le (2d-1)e^{-\eta} + e^{-2\eta} \le 2d e^{-\eta}
\]
To get the expected error (with respect to the training data), we bound
\begin{align*}
\EE[e(\w_2)] &= \EE[e(\w_2)|E] P(E) + \EE[e(\w_2)|E'] P(E')
\le \EE[e(\w_2)|E] + \delta \EE[e(\w_2)|E'] \\
&\le \EE[e(\w_2)|E] + 2 \delta 
= 2de^{-\eta} + 8 d e^{-\frac{md}{32\sigma^2}},
\end{align*}
where we used the fact that $e(\w_2) \le 2(1-v_{2,1}^+ + v_{2,1}^-) \le 4$ as $v_{2,1}^+, v_{2,1}^- \in [0,1]$, and that maximizing convex
function $e(\w)$ give $\w$ 
Thus, taking sufficiently large $\eta$, we can drop the error arbitrarily
close to $8d e^{-\frac{md}{32\sigma^2}}$.\hfill $\BlackBox$
\section{Upper bound for Approximated EGU$^{\pm}$: Proof of Theorem \ref{thm:spindly}}

\label{a:ApproxEGU-upper}

We derive the Approximated EGU$^{\pm}$ algorithm defined by \eqref{eq:spindly_update_main} 
as a first-order approximation of the 
unnormalized Exponentiated Gradient update.

The vanilla EGU$^{\pm}$ algorithm
keeps track of two vectors, $\v^+_t$ and $\v^-_t$, and the prediction vector is given by $\w_t = \v^+_t - \v^-_t$. 
Let $\beta$ denote the initial value of weights, that is $\v_1^+ = \v_1^- = \beta \1$. The weights
are updated according to
\begin{equation}
\v^{\pm}_{t+1} = \v^{\pm}_t \odot e^{\mp \eta \nabla \emploss(\w_t)}
= \beta e^{\mp \eta \sum_{j=1}^t \nabla \emploss(\w_j)}.
\label{eq:egupm_update}
\end{equation}
At every timestamp we have $\v^+_t \v^-_t = \beta^2$, which together with
$\w_t = \v^+_t - \v^-_t$, allows us to express $\v^+_t and \v^-_t$ in terms of $\w_t$:
\begin{equation}
\v^+_t = \frac{\w_t + \sqrt{\w_t^2 + 4\beta^2}}{2}, \qquad 
\v^-_t = \frac{-\w_t + \sqrt{\w_t^2 + 4\beta^2}}{2}.
\label{eq:v_plus_v_minus_update}
\end{equation}
Expanding the EGU$^{\pm}$ update \eqref{eq:egupm_update} in the learning rate we get
\[
\v_{t+1}^{\pm} = \v_t^{\pm} e^{\mp\eta \nabla L(\w_t)}
= \v_t^{\pm} (1 \mp \eta \nabla L(\w_t)) + O(\eta^2).
\]
Dropping the $O(\eta^2)$ term and using \eqref{eq:v_plus_v_minus_update} gives 
\[
\v_{t}^+ + \v_t^- = \sqrt{\w^2 + 4\beta^2}, 
\]
so that the update becomes \eqref{eq:spindly_update_main}:
\[
\w_{t+1} = \w_t - \eta\sqrt{\w_t^2 + 4\beta^2}\; \nabla L(\w_t).
\]

\noindent{\bf Proof of Theorem \ref{thm:spindly}.}
Using \eqref{eq:gradient_simplified}, $\nabla \emploss(\w_t) = 2(\w_t - \e_1 - \bm{\zeta})$, 
the update \eqref{eq:spindly_update_main} becomes:
\begin{equation}
\w_{t+1} = \w_t - 2\eta\sqrt{\w_t^2 + 4\beta^2}\; (\w_t - \e_1 - \bm{\zeta}) 
\label{eq:spindly_update}
\end{equation}

We set $\beta=1/(2d)$, $\eta = 1/4$, and the number of steps of the algorithm 
$T=4 \sqrt{d}$. Note that $d \ge 4$
Using the deviation bound for zero-mean Gaussian $z \sim N(0,\tau^2)$,
$P(z \ge -\gamma) \le \exp\left\{-\frac{\gamma^2}{2\tau^2}\right\}$, we get 
that $P(|\zeta_i| \ge \gamma) \le 2 \exp \left\{- \frac{md \gamma^2}{8 \sigma^2} \right\}$.
Taking the union bound we have
$P(\exists i \; |\zeta_i| \ge \gamma) \le 2d \exp \left\{- \frac{md \gamma^2}{8 \sigma^2} \right\}$. Denoting the probability on the left-hand side by $\delta$, we can solve
for $\gamma$:
\[
\gamma = \sigma \sqrt{\frac{\ln \frac{2d}{\delta}}{md}}.
\]
This means the with probability at least $1-\delta$, $|\zeta_i| \le \gamma$ for all $i=1,\ldots,d$. Let us call this event $E$, and we condition everything what follows on the fact that
$E$ happened. Furthermore, due to our assumption that $m$ grows at least logarithmically with $d$, $m \ge 8 \sigma^2 \ln \frac{2 d}{\delta}$, we have 
$\gamma \le \frac{1}{\sqrt{8d}}$.

We rewrite the update \eqref{eq:spindly_update} in terms of $\s_t = \w_t - \e_1 - \bm{\zeta}$
\begin{equation}
\s_{t+1} = (1 - 2\eta\sqrt{(\s_t + \e_1 + \bm{\zeta})^2 + 4\beta^2}) \s_t
\label{eq:update_s}
\end{equation}





\paragraph{The analysis for `zero signal' direction.} Since every weights evolves independently of the other weights, we can analyze each coordinate separately. We start with any coordinate
$i \ge 2$ ('zero signal' weights), for which the update becomes 
\[
s_{t+1,i} = (1 - 2\eta\sqrt{(s_{t,i} + \zeta_i)^2 + 4\beta^2}) s_{t,i},
\]
with $s_{t,1} = - \zeta_1$. Now, w.l.o.g. assume $\zeta_i < 0$ (the analysis for $\zeta_i > 0$ is analogous). This means that $s_{1,i} > 0$, and $s_{t,i}$ is positive and 
monotonically decreasing as long as $2\eta\sqrt{(s_{t,i} + \zeta_i)^2 + 4c^2} < 1$;
this condition is ensured by noticing that
\[
(s_{t,i} + \zeta_i)^2 + 4\beta^2 \le \zeta_i^2 + d^{-2} \le \gamma^2 + \frac{1}{4} \le
\frac{1}{2},
\]
so that using learning rate $\eta < \sqrt{2}/2$ will do the trick (remind that we use $\eta=1/4$).
Since we know that $s_{t,i}$ is positive and monotonically decreasing, we can bound:
\[
s_{t+1,i} \ge (1 - 2\eta\sqrt{(s_{T+1,i} + \zeta_i)^2 + 4\beta^2}) s_{t,i},
\]
so that
\[
s_{T+1,i} \ge (1 - 2\eta\sqrt{(s_{T+1,i} + \zeta_i)^2 + 4\beta^2})^T s_{1,i}
\ge (1 - 2\eta T \sqrt{(s_{T+1,i} + \zeta_i)^2 + 4\beta^2}) \zeta_i,
\]
where we used the Bernoulli inequality $(1+x)^n \ge 1 +xn$. 
Returning to the original variable $w_{t,i} = s_{t,i} + \zeta_i$ gives
\[
w_{T+1,i} \ge 2 \eta T \zeta_i \sqrt{w_{T+1,i}^2 + 4\beta^2}.
\]
Since we also know that $w_{T+1,i} < 0$ (because $s_{t,i}=w_{t,i}-\zeta_i$ was
decreasing in $t$ with $s_{1,i}=-\zeta_i$ and $w_{1,i}=0$), we have
\[
w_{T+1,i}^2 \le 4 \eta^2 T^2 \zeta_i^2 (w_{T+1,i}^2 + 4\beta^2),
\]
which can be solved for $w_{T+1,i}^2$:
\[
w_{T+1,i}^2  \le \frac{16 \beta^2 \eta^2 T^2 \zeta_i^2}{1 - 4 \eta^2 T^2 \zeta_i^2}
= \frac{1}{d^2} \frac{4\eta^2 T^2 \gamma^2}{1 - 4\eta^2 T^2 \gamma^2}
\]
This expression is increasing in $\zeta_i^2$ so we can upper-bound it by
\[
w_{T+1,i}^2  \le \frac{1}{d^2} \frac{4\eta^2 T^2 \gamma^2}{1 - 4\eta^2 T^2 \gamma^2}
\]
Since $\eta=1/4$, $T=4\sqrt{d}$ and $\gamma \le \frac{1}{\sqrt{8d}}$, the denominator
is bounded from below
\[
1 - 4\eta^2 T^2 \gamma^2 \ge 1 - \frac{1}{2} \ge \frac{1}{2},
\]
so that, using $\gamma = \sigma \sqrt{\frac{\ln \frac{2d}{\delta}}{md}}$, 
\[
w_{T+1,i}^2 \le \frac{1}{d^2} \frac{8d \cdot \sigma^2\ln \frac{2d}{\delta}}{md}
= \frac{8 \sigma^2 \ln \frac{2d}{\delta}}{m d^2}.
\]
Thus, the total error from `zero-signal' coordinates is
\begin{equation}
\sum_{i=2}^d w_{T+1,i}^2 \le \frac{8 \sigma^2 \ln \frac{2d}{\delta}}{m d}
\label{eq:error_bound_2_to_d}
\end{equation}
We also need to show that the same amount of error comes from the first coordinate.

\paragraph{Analysis for the `signal' coordinate.} 
Using \eqref{eq:update_s}, we have for $s_{t,1} = 1 - \zeta_i - w_{t,1}$:
\[
s_{t+1,1} =  (1 - 2\eta \sqrt{(1-s_{t,1}-\zeta_i)^2+4\beta^2}) s_{t,1},
\]
with $s_{1,1} = 1 - \zeta_i$. As before, we note that $s_{t,i}$ is decreasing in $t$,
as long as $2\eta \sqrt{(1-s_{t,1}-\zeta_i)^2+4\beta^2} < 1$. However, this condition is satisfied
for our choice of of $\eta=1/4$, because
\begin{align*}
2\eta \sqrt{(1-s_{t,1}-\zeta_i)^2+4\beta^2} &\le 
2\eta \sqrt{(1-\zeta_i)^2+4\beta^2} \le \frac{1}{2} \sqrt{1 + 2|\zeta_i| + \zeta_i^2 + d^{-2}}\\
&\le \frac{1}{2} \sqrt{1 + \frac{1}{\sqrt{2d}} + \frac{9}{8d^2}}
\stackrel{d \ge 4}{\le} \frac{1}{2} \sqrt{1 + \frac{1}{8} + \frac{9}{128}} < 1.
\end{align*}
Now, we need to carefully analyze the update. Initially $s_{t,1}$ decreases slowly, as
the square root term is essentially of order $1/d$. At some point, however,
$s_{t,1}$ falls below a certain constant (say, $s_{t,1} = 1/2$), and the square root
term is of order $O(1)$, and the convergence becomes exponential.

First, to simplify analysis we simply denote $s_{t,i}$ by $s_t$; moreover, define $r=1-\zeta_i$, so that the update becomes
\begin{equation}
s_{t+1}= (1- 2\eta \sqrt{(r-s_t)^2 + 4\beta^2} s_t, \qquad s_1 = r,
\label{eq:s_1_simplified}
\end{equation}

Already after the first iteration,
\[
s_2 = (1- 1/2 d^{-1})s_t = \frac{2d-1}{2d}r,
\]
so that
$(r-s_2)^2 = 1/(4d^2)$ becomes comparable with $4 \beta^2 = 1/d^2$ term. So we can drop
the $4\beta^2$ term from the square root in \eqref{eq:s_1_simplified} and upper bound
\begin{equation}
s_{t+1} \le (1-2\eta(r-s_t))s_t
\label{eq:s_1_upper_bound}
\end{equation}
To get some insight into this expression we solve the corresponding differential equation:
\[
\dot{s} = -2\eta(r-s)s,
\]
which give: 
\[
\frac{s_t}{r-s_t} = C e^{-\eta r t} \quad \Longrightarrow \quad s_t = \frac{r}{1 + C e^{2\eta r t}}.
\]
Inspired by this we will bound $\frac{s_t}{r-s_t}$. From \eqref{eq:s_1_upper_bound} we get:
\[
r-s_{t+1} \ge r - (1-2\eta(r-s_t)s_t = (1 + 2\eta s_t) (r-s_t),
\]
so that:
\[
\frac{s_{t+1}}{r-s_{t+1}} \le \underbrace{\left(\frac{1-2\eta(r-s_t)}{1+2\eta s_t}\right)}_{=:A_t} \;
\left(\frac{s_t}{r-s_t}\right).
\]
We will now bound $A_t$ independent of $s_t$. To this end, note that
$A_t$ is maximized when $s_t = r$. Indeed,
\[
A_t = \frac{1+2\eta s_t - 2\eta r}{1 + 2\eta s_t} = 1 - \frac{2\eta r}{1 + 2\eta s_t}
\le 1 - \frac{2\eta r}{1 + 2\eta r} = \frac{1}{1+2\eta r}.
\]
This way, we get an upper bound:
\[
\frac{s_{T+1}}{r-s_{T+1}} \le (1+2\eta r)^{-(T-1)} \frac{s_2}{r-s_2}
= (2d-1) (1+2\eta r)^{-(T-1)},
\]
or by solving for $s_{T+1}$,
\[
s_{T+1} \le \frac{r}{1 + (2d-1)^{-1} (1+2\eta r)^{T-1}} \le r (2d-1) (1 + 2\eta r)^{-(T-1)}.
\]
The expression above is decreasing in $r$ for $T \ge 4$ (can be verified by computing the derivative), so we we will upper-bound it by lower-bounding $r$,
that is $r = 1 - \zeta_i \ge 1 - \frac{1}{\sqrt{8d}}$.
Taking $\eta = 1/4$ and using $1 - \frac{1}{\sqrt{8d}} 
\stackrel{d \ge 4}{\ge} 1 - \frac{1}{4 \sqrt{2}}$,
we get $1 + 2\eta r >= \frac{3}{2} - \frac{1}{8 \sqrt{2}} > e^{1/3}$ (checked numerically).
Therefore,
\[
s_{T+1} \le (2d-1) e^{-(T-1)/3}
\le e^{1/3 + \ln 2} e^{-T/3 + \ln d} \le 3 e^{-T/3 + \ln d}.
\]
Using the fact that $T=4\sqrt{d}$, we get
\[
s_{T+1} \le 3 e^{-\frac{4}{3} \sqrt{d} + \ln d} 
\]
To bound $(1-w_{T+1,1})^2$ we use
\begin{equation}
(1-w_{T+1,1})^2 = (s_{T+1,1}+\zeta_1)^2 \le
2s_{T+1,1}^2 + 2 \zeta_1^2
\le 9 e^{-\frac{8}{3} \sqrt{d} + 2 \ln d} + \frac{2 \sigma^2 \ln \frac{2d}{\delta}}{md}.
\label{eq:error_bound_1}
\end{equation}

\paragraph{Bound the error of Approximated EGU${}^{\pm}$ algorithm}

The final error of the algorithm is obtained by summing \eqref{eq:error_bound_2_to_d} and \eqref{eq:error_bound_1}:
\begin{align*}
\|\w_{T+1} - \e_1\|^2
&\le 9 e^{-\frac{8}{3} \sqrt{d} + 2 \ln d} + \frac{2 \sigma^2 \ln \frac{2d}{\delta}}{md} + \frac{8 \sigma^2 \ln \frac{2d}{\delta}}{m d} \\
&= \frac{10 \sigma^2 \ln \frac{2d}{\delta}}{m d} 
+ 9 e^{-\frac{8}{3} \sqrt{d} + 2 \ln d}.
\end{align*}

\vspace{-.9cm}
\hfill $\BlackBox$

\section{Upper bound for the spindly network} \label{app:spindly}

In this section, we upper-bound the error of the spindly network defined in
Figure \ref{f:spindly}, showing essentially the same bound (up to constants) as for the Approximated
EGU$^{\pm}$. We note that essentially the same algorithm has been analyzed by
\citet{optimal}, giving the same bound $O(\frac{\sigma^2\log d}{md})$ under the
restricted isometry property (RIP) assumption.

\begin{theorem}
Assume $d \ge 4$ is such that $\sqrt{d}$ is an integer. Consider the spindly network given in Figure \ref{f:spindly} trained with gradient descent, with weights initialized as $\u = \sqrt{2/d} \bm{1}$ and $\v = \bm{0}$, and the learning rate set to $\eta=1/4$. Let $m \ge 8 \sigma^2 \ln \frac{2 d}{\delta} = \Omega(\sigma^2 \log (d/\delta))$ . With probability at least $1-\delta$, the algorithm run for $T=4\sqrt{d}$ steps achieves error bounded by:
\[
e(\w_{T+1}) \le
\frac{33 \sigma^2 \ln \frac{2d}{\delta}}{4 m d} 
+ 16 e^{-\frac{8}{3} \sqrt{d} + 2 \ln d}
= O\left(\frac{\sigma^2\log d}{md} + e^{-\frac83\sqrt{d}}\right)
\]
\end{theorem}

\begin{proof} The spindly network predicts with $\w_t$ given by $\w_t = \u_t \odot \v_t$, and both
$\u_t$ and $\v_t$ are updated with the gradient descent algorithm:
\[
\u_{t+1} = \u_t - \eta \nabla_{\u_t} L(\w_t), \qquad
\v_{t+1} = \v_t - \eta \nabla_{\v_t} L(\w_t)
\]
Using \eqref{eq:gradient_simplified} and the chain rule, $\nabla_{\u_t} \emploss(\w_t) = 2(\w_t - \e_1 - \bm{\zeta}) \odot \v_t$ and $\nabla_{\v_t} \emploss(\w_t) = 2(\w_t - \e_1 - \bm{\zeta}) \odot \u_t$, so that
\begin{equation}
\u_{t+1} = \u_t - 2\eta (\w_t - \e_1 - \bm{\zeta}) \odot \v_t, \qquad
\v_{t+1} = \v_t - 2\eta (\w_t - \e_1 - \bm{\zeta}) \odot \u_t
\label{eq:uv_update}
\end{equation}
Let us introduce two vectors, 
$\v^+_t$ and $\v^-_t$, given by:
\[
\v^+_t = \frac{1}{4} (\u_t + \v_t)^2, \qquad \v^-_t = \frac{1}{4} (\u_t - \v_t)^2,
\]
(square applied coordinatewise) and note that
\[
\w_t = \u_t \cdot \v_t = \v^+_t - \v^-_t.
\]
Subtracting and adding equations in \eqref{eq:uv_update}, followed by squaring both sides
gives
\begin{align}
\v^+_{t+1} &= (1 - 2 \eta (\w_t - \e_1 - \bm{\zeta}))^2 \odot \v^+_t \nonumber \\
\v^-_{t+1} &= (1 + 2 \eta (\w_t - \e_1 - \bm{\zeta}))^2 \odot \v^-_t  \label{eq:v_pm_spindly}
\end{align}
We initialize the algorithm as follows:
\[
\u_1 = \sqrt{\frac{2}{d}}, \qquad \v_1 = \bm{0},
\]
so that $\v_1^+ = \v_1^- = \frac{1}{2d} \bm{1}$ and $\w_1 = \bm{0}$, similarly
as in the Approximated EGU${}^{\pm}$. We set $\eta = \frac{1}{8}$, $T=4\sqrt{d}$,
and use the same assumption as in the previous section, that is $m \ge 8 \sigma^2 \ln \frac{2d}{\delta}$, which implies that with probability 
 at least $1-\delta$, $|\zeta_i| \le \gamma$ for all $i=1,\ldots,d$, where
\begin{equation}
\gamma = \sigma \sqrt{\frac{\ln \frac{2d}{\delta}}{md}} \le \frac{1}{\sqrt{8d}}.
\label{eq:upper_bound_on_gamma}
\end{equation}
As before, we denote the high probability event above as $E$, and we condition everything what follows on the fact that
$E$ happened. We will also assume that $d \ge 9$.

\paragraph{The analysis for `zero signal' direction.} As before, we can analyze each coordinate separately. We start with any coordinate
$i \ge 2$, for which the update \eqref{eq:v_pm_spindly} becomes 
\begin{equation}
v^{\pm}_{t+1,i} = (1 \mp 2\eta(w_{t,i} - \zeta_i))^2 v^{\pm}_{t,i}
\label{eq:v_pm_update_zero_signal}
\end{equation}
W.l.o.g. assume $\zeta_i > 0$ (the analysis for $\zeta_i < 0$ is analogous). We will first
prove by induction on $t$ that $0 \le w_{t,i} \le \zeta_i$ for all $t=1,\ldots,T+1$. Since $w_{1,i}=0$ and $\zeta_i > 0$, it clearly holds for $t=1$. Now, assume that it holds for iterations $1,\ldots,t$ and we prove that it also holds for $t+1$. We have:
\begin{align}
w_{t+1,i} &= v_{t+1,i}^+ - v_{t+1,i}^- \nonumber \\
&= (1 - 2\eta(w_{t,i} - \zeta_i))^2 v_{t,i}^+ -
(1 + 2\eta(w_{t,i} - \zeta_i))^2 v_{t,i}^- \nonumber \\
&= (1+4 \eta^2(w_{t,i} - \zeta_i)^2) w_{t,i}
- 4 \eta (w_{t,i} - \zeta_i) (v^+_{t,i} + v^-_{t,i}) \ge 0, \label{eq:deriv_spindly_induction}
\end{align}
because  $w_{t,i} - \zeta_i < 0$ and $w_{t,i} \ge 0$ from the induction assumption, while
$v^+_{t,i}, v^-_{t,i} \ge 0$ from their definitions.

Now, since $-\zeta_i \le w_{q,i} - \zeta_i \le 0$ for $q=1,\ldots,t$ from the inductive assumption, 
$\eta = \frac{1}{8}$ and $|\zeta_i| \le \gamma \le \frac{1}{\sqrt{8d}}$ from \eqref{eq:upper_bound_on_gamma}, we have $0 \ge 2 \eta (w_{q,i} - \zeta_i) \ge - \frac{1}{4 \sqrt{8d}}$,
so that
\[
(1 + 2\eta(w_{q,i} - \zeta_i))^2 < 1, \qquad
(1 - 2\eta(w_{q,i} - \zeta_i))^2 > 1, \qquad q=1,\ldots,t.
\]
Thus, we see from \eqref{eq:v_pm_update_zero_signal} that $v^+_{q,i}$ is monotonically
increasing in $q$, and $v^-_{q,i}$ is monotonically decreasing in $q$. This means that
$w_{q,i} = v^+_{q,i} - v^-_{q,i}$ is monotonically increasing in $q$. Therefore,
\[
(1 + 2\eta(w_{q,i} - \zeta_i))^2 \ge (1 + 2\eta(w_{1,i} - \zeta_i))^2
= (1 - 2 \eta \zeta_i)^2 \qquad
\text{for all~} q=1,\ldots,t,
\]
and, similarly,
\[
(1 - 2\eta(w_{q,i} - \zeta_i))^2 \le (1 + 2\eta(w_{1,i} - \zeta_i))^2 = (1 + 2\eta \zeta_i)^2 \qquad
\text{for all~} q=1,\ldots,t.
\]
This let us upper-bound $v^+_{t+1,i}$ as
\[
v^+_{t+1,i} = \prod_{q=1}^t (1 - 2 \eta(w_{q,i} - \zeta_i))^2 v^+_{1,i}
\le (1 + 2 \eta \zeta_i)^{2t} v^+_{1,i} =
(1 + 2 \eta \zeta_i)^{2t} \frac{1}{2d}, 
\]
and similarly lower-bound $v^-_{t+1,i}$ as 
\[
v^-_{t+1,i} = \prod_{q=1}^t (1 + 2 \eta(w_{q,i} - \zeta_i))^2 v^-_{1,i}
\ge (1 - 2 \eta \zeta_i)^{2t} v^+_{1,i} =
(1 - 2 \eta \zeta_i)^{2t} \frac{1}{2d}.
\]
We have
\[
(1+2 \eta \zeta_i)^{2t} = e^{2t \ln(1+ 2\eta \zeta_i)} \le e^{4 t \eta \zeta_i},
\]
where we used $\ln(1+x) \le x$. Moreover, by Bernoulli's inequality,
\[
(1 - 2 \eta \zeta_i)^{2t} \ge 1 - 4t \eta \zeta_i.
\]
This gives:
\[
w_{t+1,i} = 
v^+_{t+1,i} - v^-_{t+1,i} 
=  \frac{1}{2d} \left( (1 + 2 \eta \zeta_i)^{2t} - (1 - 2 \eta \zeta_i)^{2t}  \right)
\le \frac{1}{2d} (e^{4 t \eta \zeta_i} - 1 + 4 t \eta \zeta_i).
\]
Now, note that
\[
4 t \eta \zeta_i \le \frac{T}{2} \gamma \le \frac{4 \sqrt{d}}{2 \sqrt{8 d}} \le \frac{1}{\sqrt{2}},
\]
Using the convexity of $e^x$, we have for $x \in [0,a]$ that
\[
e^x = e^{\left(1-\frac{x}{a}\right) \cdot 0 + \frac{x}{a} \cdot a} \le \left(1-\frac{x}{a}\right) e^0 + \frac{x}{a} e^a
= 1 + x \frac{e^a -1}{a}. 
\]
Taking $x = 4 t \eta \zeta_i$ and $a = \frac{1}{\sqrt{2}}$, we have
\[
e^{4 t \eta \zeta_i} 
\le 1 + 4 t \eta \zeta_i \sqrt{2} (e^{1/\sqrt{2}} - 1) \le 1 + 6 t \eta \zeta_i.
\]
This allows us to bound
\begin{equation}
w_{t+1,i} \le \frac{1}{2d} (e^{4 t \eta \zeta_i} - 1 + 4 t \eta \zeta_i)
\le \frac{1}{2d} 10 t \eta \zeta_i \le \frac{5 T \eta \zeta_i}{d}
= \frac{5}{2  \sqrt{d}} \zeta_i \le \frac{5}{6} \zeta_i \le \zeta_i,
\label{eq:zero_signal_useful_inequality}
\end{equation}
where we used $d \ge 9$. This finishes the inductive proof that $0 \le w_{t,i} \le \zeta_i$
for all $t = 1,\ldots,T+1$. However, applying \eqref{eq:zero_signal_useful_inequality} to $t=T$ gives
\[
w_{T+1,i} \le  \frac{5}{2  \sqrt{d}} \zeta_i,
\]
so that using \eqref{eq:upper_bound_on_gamma}
\[
w^2_{T+1,i} \le  \frac{25}{4 d} \zeta^2_i
\le \frac{25}{4 d} \gamma^2
=\frac{25 \sigma^2 \ln \frac{2d}{\delta}}{4 md^2}.
\]
Thus, the total error from `zero-signal' coordinates is
\begin{equation}
\sum_{i=2}^d w_{T+1,i}^2 \le \frac{25 \sigma^2 \ln \frac{2d}{\delta}}{4 m d}.
\label{eq:spindly_error_bound_for_i_ge_2}
\end{equation}

\paragraph{Analysis for the `signal' coordinate.} 
For $i=1$, the update becomes:
\[
v^{\pm}_{t+1,1} = (1 \mp 2\eta(w_{t,i} - 1 - \zeta_1))^2 v^{\pm}_{t,1}.
\]
First, we will show by induction that $0 \le w_{t,1} \le 1 + \zeta_1$ for
all $t=1,\ldots,T+1$. 
This is clearly true for $t=1$ as $w_{t,1} = 0$ and $|\zeta_1| \le \frac{1}{\sqrt{8d}}$. Assume now this is true for $q=1,\ldots,t$, and we prove it is also true for $t+1$.
We have
\begin{align}
w_{t+1,1} &= v_{t+1,1}^+ - v_{t+1,1}^- \nonumber \\
&= (1 - 2\eta(w_{t,1} - 1 - \zeta_1))^2 v_{t,1}^+ -
(1 + 2\eta(w_{t,1} - 1 - \zeta_1))^2 v_{t,1}^- \nonumber \\
&= (1+4 \eta^2(w_{t,1} - 1 - \zeta_1)^2) w_{t,1}
- 4 \eta (w_{t,1} - 1 - \zeta_1) (v^+_{t,1} + v^-_{t,1}) \nonumber \\
&= (1+4 \eta^2(w_{t,1} - 1 - \zeta_1)^2) w_{t,1}
- 4 \eta (w_{t,1} - 1 - \zeta_1) (w_{t,1} + 2 v^-_{t,1}) \nonumber \\
&= (1 - 2\eta(w_{t,1} - 1 - \zeta_1))^2 w_{t,1} + 8 \eta (1 + \zeta_i - w_{t,1}) v^-_{t,1}.
\label{eq:expression_for_w_t_plus_1_signal}
\end{align}
It follows from the induction assumption that both terms in the last line are nonnegative,
thus $w_{t+1,1} \ge 0$. To show that $w_{t+1,1} \le 1 + \zeta_1$, 
note that by inductive assumption 
$1 + 2\eta(w_{t,i} - 1 - \zeta_1)) \ge 1 - 2\eta(1- \zeta_1) 
\ge 1 - 2\eta(1 + \gamma) = \frac{3}{4} - \frac{1}{4} \gamma \ge
\frac{3}{4} - \frac{1}{4\sqrt{8}{d}} > 0$, and also
$1 + 2\eta(w_{t,i} - 1 - \zeta_1)) \le 1$. This means that $v^{-}_{q,1}$ is nonincreasing
in $q$, and thus $v^{-}_{t,1} \le v^{-}_{1,1} = \frac{1}{2d} \le \frac{1}{18}$ (due to $d \ge 9$). Plugging this into \eqref{eq:expression_for_w_t_plus_1_signal}, we can bound
\begin{align}
w_{t+1,1} &\le 
(1 - 2\eta(w_{t,1} - 1 - \zeta_1))^2 w_{t,1} + 8 \eta (1 + \zeta_1 - w_{t,1}) \frac{1}{18}
\nonumber \\
&= (1 - 2\eta(w_{t,1} - 1 - \zeta_1))^2 w_{t,1} + \frac{4}{9} \eta (1 + \zeta_1 - w_{t,1}).
\label{eq:signal_upper_bound_w_t_plus_1}
\end{align}
We now maximize
the right-hand side of \eqref{eq:signal_upper_bound_w_t_plus_1} with respect $w_{t,1} \in [0, 1 + \zeta_1]$.
To this end, define function
\[
f(x) = (1+(a-x)/4)^2 x + (a-x) / 18,
\]
with $a=1+\zeta_1$. We get
\[
f'(x) = - \frac{x}{2} (1 + (a-x)/4) + (1+(a-x)/4)^2 - \frac{1}{18},
\]
which is a convex quadratic function of $x$, so that $f(x)$ achieves its maximum
in the left root of $f'(x)$. Solving $f'(x) = 0$ is equivalent to
\[
\frac{3x^2}{16} -x\left(1 + \frac{a}{4}\right) + \frac{17}{18} + \frac{a}{2} + \frac{a^2}{16} = 0,
\]
which left root is given by
\[
x_{\ell} = \frac{1}{3} \left(2a + 8 - \sqrt{a^2 + 8a + \frac{51}{3}}\right).
\]
Note that $a = 1 + \zeta_1 \ge 1 - \gamma > 1 - \frac{1}{\sqrt{8d}} \ge 1 - \frac{1}{3 \sqrt{8}} \ge 0.88 := a_0$. One can verify that $x_{\ell}$ is increasing in $a$ (e.g. by inspecting the sign of the derivative of $x_{\ell}$ with respect to $a$), which means that
\[
x_{\ell} \ge 
\frac{1}{3} \left(2a_0 + 8 - \sqrt{a_0^2 + 8a_0 + \frac{51}{3}}\right)
\ge 1.59.
\]
This value is to the right of range $[0,a]$, because
$a \le 1 + \gamma > 1 + \frac{1}{3 \sqrt{8}} \le 1.12$. This means that the maximum of
$f(x)$ in the range $[0,a]$ is achieved for $x=0$. This correspoinds to $w_{t,1}=1+\zeta_i$
in \eqref{eq:signal_upper_bound_w_t_plus_1}, which gives:
\[
w_{t+1,1} \le 
1 + \zeta_1,
\]
which was to be shown by induction. 

We now lower bound $w_{t+1,1}$. Using \eqref{eq:expression_for_w_t_plus_1_signal} and
the proven fact that $1+\zeta_i - w_{t,1} \ge 0$ for all $t$, we have
\[
w_{t+1,1} \ge (1 - 2\eta(w_{t,1} - 1 - \zeta_1))^2 w_{t,1}
= (1 + 2\eta(r - w_{t,1}))^2 w_{t,1}.
\]
where we simplified the notation with $r = 1 + \zeta_1$. We further bound
\[
(1 + 2\eta(r - w_{t,1}))^2
= 
(1 + 4\eta(r - w_{t,1}) + 4\eta^2(r - w_{t,1})^2)
\ge 1 + 4 \eta (r - w_{t,1}),
\]
so that
\[
w_{t+1,1}  \ge (1 + 4 \eta (r-w_{t,1})) w_{t,1}.
\]
Now consider expression $Q_{t+1} = \frac{r}{w_{t+1,1}} - 1$. Clearly, $Q_{t+1}$
is decreasing $w_{t+1,1}$ so we have
\begin{align*}
Q_{t+1} &\le \frac{r}{(1 + 4 \eta (r - w_{t,1})) w_{t,1}} - 1
= \frac{Q_t + 1}{1 + 4 \eta (r-w_{t,1})} - 1 \\
&= \frac{Q_t - 4 \eta (r - w_{t,1})}{1 + 4 \eta (r-w_{t,1})}
= \frac{Q_t - 4 \eta w_{t,1} Q_t}{1 + 4 \eta (r-w_{t,1})}
= Q_t \frac{1 - 4 \eta w_{t,1}}{1 + 4 \eta (r-w_{t,1})}.
\end{align*}
Now, note that
\[
\frac{1 - 4 \eta w_{t,1}}{1 + 4 \eta (r-w_{t,1})}
= 1 - \frac{4 \eta r}{1 + 4 \eta (r-w_{t,1})}
\le 1 - \frac{4 \eta r}{1 + 4 \eta r} = \frac{1}{1 + 4 \eta r},
\]
where we used $w_{t,1} \ge 0$. Thus we get
\[
Q_{t+1} \le  \frac{1}{1 + 4 \eta r} Q_t
\]
for all $t=1,\ldots,T$, which implies
\[
Q_{T+1} \le \frac{1}{(1 + 4 \eta r)^{T-1}} Q_2
\]
(we cannot start from $Q_1$ as it is undefined). To obtain $Q_2$, we note
that
\[
w_{2,1} = (1 + 2 \eta r)^2 v^+_1 - (1 - 2 \eta r)^2 v^-_1
= \frac{1}{2d} \left( (1 + 2 \eta r)^2 - (1 - 2 \eta r)^2 \right)
= \frac{4 \eta r}{d}.
\]
Therefore,
\[
Q_2 = \frac{d}{4 \eta} - 1 = 2d - 1 \le 2d,
\]
and so
\[
Q_{T+1} = \frac{r - w_{T+1,1}}{w_{T+1,1}} \le \frac{2d}{(1 + r/2)^{T-1}},
\]
or, equivalently,
\[
r - w_{T+1,1} \le r \frac{\frac{2d}{(1 + r/2)^{T-1}}}{1 +\frac{2d}{(1 + r/2)^{T-1}} }
\le r  \frac{2d}{(1 + r/2)^{T-1}}.
\]
We can bound $r = 1 + \zeta_1 \ge 1 - \gamma \ge 1 - \frac{1}{\sqrt{8d}}
\ge 1 - \frac{1}{3 \sqrt{8}} \ge 0.88$ and similarly
$r \le 1 + \frac{1}{3 \sqrt{8}} \le 1.12$. This gives
$1 + r/2 \ge 1.44 \ge e^{1/3}$, so that
\[
r - w_{T+1,1} \le 2.24 d e^{-1/3 (T-1)} = 2.24 e^{1/3} e^{-4/3 \sqrt{d} + \ln d}
\le 4 e^{-4/3 \sqrt{d} + \ln d}
\]
To bound $(1-w_{T+1,1})^2$ we use
\begin{equation}
(1-w_{T+1,1})^2 = (r - w_{T+1,i}-\zeta_1)^2 \le
2(r - w_{T+1,1})^2 + 2 \zeta_1^2
\le 16 e^{-8/3 \sqrt{d} + 2\ln d} + \frac{2 \sigma^2 \ln \frac{2d}{\delta}}{md},
\label{eq:spindly_error_bound_1}
\end{equation}
where in the last line we used \eqref{eq:upper_bound_on_gamma}.

\paragraph{Bound the error of the Spindly network.}
The final error of the algorithm is obtained by summing
\eqref{eq:spindly_error_bound_for_i_ge_2} and \eqref{eq:spindly_error_bound_1}:
\begin{align*}
\|\w_{T+1} - \e_1\|^2
&\le 16 e^{-8/3 \sqrt{d} + 2\ln d} + \frac{2 \sigma^2 \ln \frac{2d}{\delta}}{md}
+ \frac{25 \sigma^2 \ln \frac{2d}{\delta}}{4 m d} \\
&= \frac{33 \sigma^2 \ln \frac{2d}{\delta}}{4 m d} 
+ 16 e^{-\frac{8}{3} \sqrt{d} + 2 \ln d}.
\end{align*}
\end{proof}

\section{Upper bound for the priming method}
\label{app:priming}
The priming method operates as follows \citep{priming}: 
First, it computes the least squares estimator
$\predw_{LS} = (\X^\top \X) \X^\top \y$. Then, it scales 
each column of $\X$ (each feature)
by the corresponding coordinate of $\predw_{LS}$, resulting in a new matrix
$\widetilde{\X} = \X \mathrm{diag}(\predw_{LS})$. Next, it calculates the Ridge Regression solution $\widetilde{\w}_{RR}$ 
using the new inputs $\widetilde{\X}$ and an appropriate regularization constant $\lambda$. The final priming predictor $\predw'$ is obtained by by multiplying it
by the coordinates of $\predw_{LS}$, that is $\predw' = \mathrm{diag}(\predw_{LS}) \widetilde{\X}$.
In the proof we rewrite the priming predictor $\predw'$ as a regularized least-squares
solution, with the square regularizer based on a matrix $\lambda
\mathrm{diag}(\predw_{LS})^{-2}$. Note that the regularization strength is
amplified along directions where the coordinates of $\predw_{LS}$ are small in magnitude, effectively biasing the algorithm towards sparse solutions. 
The proof carefully bounding the expected error of such a predictor.

\begin{theorem}
Consider the priming method equipped with $\lambda = \sigma^2 \sqrt{d}$.
The expected error can be bounded by:
\[
e(\predw') \le 
\frac{17 \sigma^2}{md} + \frac{32 \sigma^4}{m^2 d}
+  \frac{4\sigma e^{-md/(8 \sigma^2)}}{\sqrt{2 \pi md}}.
\]
\label{thm:priming}
\end{theorem}
\vspace{-5mm}
Note that this upper bound
is by a factor of
$O(\log d)$ \emph{better} (assuming $\sigma^2 = O(1)$) 
than the upper bound for Approximated EGU${}^{\pm}$ (and thus
by a factor of $O(d)$ better than the error of any rotation invariant algorithm).
However we don't know whether such an improved upper bound
is also possible for Approximated EGU$^\pm$.

\vspace{2mm}
\begin{proof}
We start with rewriting the priming method into a form which is easier to analyze.
Let $\predw_{LS} = (\X^\top \X) \X^\top \y$ be the least-squares solution, which induces
a diagonal weight matrix $\W = \mathrm{diag}(\predw_{LS}$. The new (rescaled) input matrix 
is then given by $\widetilde{\X} = \X \W$. The ridge regression solution
on the new inputs with regularization constant $\lambda$ is then
$\widetilde{\w} = (\widetilde{\X}^\top \widetilde{\X} + \lambda \I_d)^{-1} \widetilde{\X}^\top \y$. Finally the output of the algorithm is $\predw' = \W \widetilde{\w}$. We thus have
\begin{align*}
\predw' &= \W (\widetilde{\X}^\top \widetilde{\X} + \lambda \I_d)^{-1} \widetilde{\X}^\top \y \\
&= \W (\W \X^\top \X \W + \lambda \I_d)^{-1} \W \X \y \\
&= \left(\W^{-1} (\W \X^\top \X \W + \lambda \I_d) \W^{-1}\right)^{-1} \X \y \\
&= \left(\X^\top \X + \lambda \W^{-2} \right)^{-1} \X \y
\end{align*}
(in any of the elements of $\predw_{LS}$ is zero, take the limit of the expression above).
Thus, $\predw'$ is a regularized least-square solution with
the quadratic regularization matrix $\lambda \W^{-2}$:
\[
\predw' = \argmin_{\predw} \left\{ \frac{1}{2} \|\y - \X \predw\|^2 + \frac{\lambda}{2} \predw^\top \W^{-2} \predw\right\}.
\]
Note that the regularization strength is
amplified along directions where the coordinates of $\predw_{LS}$ are small in
magnitude, effectively biasing the algorithm towards sparse solutions. 

We now specialize the priming predictor to our setup. We have
\[
\X^\top \X = n \I_d, \quad \text{and} \quad \X^\top \y = \X^\top \X \e_1 + \X^\top \bm{\xi}
= n \e_1 + \sqrt{n} \bm{\zeta},
\]
where $\underset{d}{\bm{\zeta}} = n^{-1/2} \X^\top \bm{\xi}$. 
Being a linear function of $\bm{\xi}$, $\bm{\zeta}$ has a normal distribution
with parameters that can be obtained from:
\[
\EE[\bm{\zeta}] = n^{-1/2} \X^\top \EE[\bm{\xi}] = \bm{0},\qquad
\EE[\bm{\zeta}\bm{\zeta}^\top] = n^{-1} \X^\top \EE[\bm{\xi} \bm{\xi}] \X
= n^{-1} \sigma^2 \X^\top \X = \sigma^2 \I_d,
\]
that is $\bm{\zeta} \sim \mathcal{N}(\bm{0}, \sigma^2 \I_d)$.
This gives us the expression for $\predw_{LS}$:
\[
\predw_{LS} = (\X^\top \X) \X^\top \y = n^{-1} (n \e_1 + \sqrt{n} \bm{\zeta})
= \e_1 + n^{-1/2} \bm{\zeta},
\]
as well as for $\predw'$:
\[
\predw' = \left(\X^\top \X + \lambda \W^{-2} \right)^{-1} \X \y
= (n\I_d + \lambda \W^{-2})^{-1} (n \e_1 + \sqrt{n} \bm{\zeta})
\]
We will analyze $\predw'$ separately for each coordinate $i=1,\ldots,d$. To simplify the notation, let $w_i$ be the $i$-th coordinate of $\predw'$, and let $v_i$ be the $i$-th coordinate of $\predw_{LS}$. The error of $\predw'$ is given by:
\[
e(\predw') = \|\predw' - \e_1\|^2 = (w_1-1)^2 + \sum_{i=2}^d w_i^2.
\]
As $\W = \mathrm{diag}(\predw_{LS})$ is
diagonal, we have for any $i > 1$
\[
w_i = \frac{\sqrt{n} \zeta_i}{n + \lambda v_i^{-2}} = \frac{\sqrt{n} v_i^2 \zeta_i}{n v_i^2 + \lambda}
= \frac{1}{\sqrt{n}} \frac{\zeta_i^3}{\zeta_i^2 + \lambda},
\]
where we used $v_i = n^{-1/2} \zeta_i$ for $i > 1$. Thus the error on the $i$-th coordinate is
\[
w_i^2 = \frac{1}{n} \frac{\zeta_i^6}{(\zeta_i^2 + \lambda)^2}
\le \frac{1}{n} \frac{\zeta_i^6}{\lambda^2}.
\]
Taking expectation over $\zeta_i \sim \mathcal{N}(0, \sigma^2)$ and using
$\EE[\zeta_i^6] = 15 \sigma^6$,
\[
\EE[w_i^2] \le \frac{15 \sigma^6}{n \lambda^2}.
\]
We now switch to coordinate $i=1$. We have $v_1 = 1 + n^{-1/2} \zeta_1
= n^{-1}(n + \sqrt{n} \zeta_1)$
so that
\[
w_1 = \frac{n + \sqrt{n} \zeta_1}{n + \lambda v_1^2}
= \frac{v_1^2 (n + \sqrt{n} \zeta_1)}{v_1^2 n + \lambda}
= \frac{n^{-2} (n+ \sqrt{n} \zeta_1)^3}{n^{-1} (n+ \sqrt{n} \zeta_1)^2 + \lambda}
= \frac{(n+ \sqrt{n} \zeta_1)^3}{n (n+ \sqrt{n} \zeta_1)^2 + n^2 \lambda}.
\]
The error on the first coordinate is thus
\begin{align*}
(w_1-1)^2 
&= \left(\frac{(n+ \sqrt{n} \zeta_1)^3 - n(n+\sqrt{n} \zeta_1)^2 - n^2 \lambda}{n (n+ \sqrt{n} \zeta_1)^2 + n^2 \lambda} \right)^2 \\
&= \left(\frac{\sqrt{n}\zeta_1 (n+ \sqrt{n} \zeta_1)^2 - n^2 \lambda}{n (n+ \sqrt{n} \zeta_1)^2 + n^2 \lambda} \right)^2.
\end{align*}
Using $(a+b) \le 2a^2 + 2b^2$ in the numerator and $(a+b) \ge a^2 + b^2$ for $a,b \ge 0$
in the denominator, we bound
\begin{align*}
(w_1-1)^2 
&\le 
2 \frac{n \zeta_1^2 (n+ \sqrt{n} \zeta_1)^4 + n^4 \lambda^2}{(n (n+ \sqrt{n} \zeta_1)^2 + n^2 \lambda)^2} 
\le \frac{2 n \zeta_1^2 (n+\sqrt{n} \zeta_1)^4}{n^2 (n + \sqrt{n} \zeta_1)^4}
+ \frac{2 n^4 \lambda^2}{n^2 (n + \sqrt{n} \zeta_1)^4 + n^4 \lambda^2} \\
&= \frac{2 \zeta_1^2}{n} + \frac{2 \lambda^2}{(\sqrt{n} + \zeta_1)^4 + \lambda^2}
\end{align*}
We now take expectation with respect $\zeta_1$ and get:
\[
\EE[(w_1-1)^2]
= \frac{2 \sigma^2}{n} + \EE\left[2 \lambda^2 ((\sqrt{n} + \zeta_1)^4 + \lambda^2)^{-1}\right].
\]
The second term requires more work. Using
$f(\zeta_1) = 2 \lambda^2 ((\sqrt{n} + \zeta_1)^4 + \lambda^2)^{-1}$:
\begin{align*}
\EE[f(\zeta_1)] &= \EE[f(\zeta_1)|\zeta_1 \le c] \, P(\zeta_1 \le c)
+ \EE[f(\zeta_1)|\zeta_1 \ge c] \, P(\zeta_1 \ge c)) \\
&\le P(\zeta_1 \le c) \max_{\zeta_1 \le c}\{f(\zeta_1)\}
+ \max_{\zeta_1 \ge c}\{f(\zeta_1)\}.
\end{align*}
We take $c = -\sqrt{n}/4$ and use a bound $P(Z < -t) \le \frac{\exp\{-t^2/2\}}{\sqrt{2\pi} t}$
to get
\[
P(\zeta_1 \le - \sqrt{n}/4) = P(Z \le - \sqrt{n}/(4\sigma))
\le \frac{2\sigma e^{-n/(32 \sigma^2)}}{\sqrt{2 \pi n}}.
\]
Moreover, 
\[
\max_{\zeta_1 \le c}\{f(\zeta_1)\} \stackrel{\zeta_1 = -\sqrt{n}}{=} = 2, \qquad
\max_{\zeta_1 \ge c}\{f(\zeta_1)\} \stackrel{\zeta_1 = c}{=} 
= \frac{2 \lambda^2}{\left(\frac{3}{4}\right)^4 n^2 + \lambda^2}
\le \frac{2 \lambda^2}{\left(\frac{3}{4}\right)^4 n^2}
\le \frac{7 \lambda^2}{n^4}
\]
Thus,
\[
\EE[(w_1-1)^2]
\le \frac{2 \sigma^2}{n} 
+ \frac{7 \lambda^2}{n^2} +
\frac{4\sigma e^{-n/(32 \sigma^2)}}{\sqrt{2 \pi n}}.
\]
Taking it all together, we have
\[
\EE[\|\predw' - \e_1\|^2]
\le (d-1)\frac{15 \sigma^6}{n \lambda^2} 
+ \frac{2 \sigma^2}{n} 
+ \frac{7 \lambda^2}{n^2} +
\frac{4\sigma e^{-n/(32 \sigma^2)}}{\sqrt{2 \pi n}}.
\]
Without optimizing too much, we simply take
$\lambda^2 = \sqrt{dn} \sigma^3 = d \sqrt{m} \sigma^3$ and get
\begin{align*}
\EE[e(\predw')]
&= \EE[\|\predw' - \e_1\|^2]
\le \frac{2 \sigma^2}{n} 
+ \frac{22 \sigma^3 \sqrt{d}}{n^{3/2}}  
+  \frac{4\sigma e^{-n/(32 \sigma^2)}}{\sqrt{2 \pi n}} \\
&= \frac{2 \sigma^2}{md} + \frac{22 \sigma^4}{m^{3/2} d}
+  \frac{4\sigma e^{-md/(32 \sigma^2)}}{\sqrt{2 \pi md}}.
\end{align*}

\vspace{-9mm}
\end{proof}

\section{Theorem \ref{t:equivalence} and EGU equivalences}
\label{a:equivalence}

\begin{proof} ({\bf Theorem \ref{t:equivalence}})

For reparameterization we have
\begin{align*}
\dot{\w} &= \frac{\partial \w}{\partial\widehat{\w}} \dot{\widehat{\w}} \\
&= - \frac{\partial \w}{\partial\widehat{\w}} \nabla_{\widehat{\w}} L \\
&= - \frac{\partial \w}{\partial\widehat{\w}} \left(\frac{\partial \w}{\partial\widehat{\w}}\right)^\top \nabla_{\w} L 
\end{align*}
so the preconditioner is $\frac{\partial \w}{\partial\widehat{\w}} \left(\frac{\partial \w}{\partial\widehat{\w}}\right)^\top$.

For MD we have
\begin{align*}
\dot{\w} &= \frac{\partial \w}{\partial\widetilde{\w}} \dot{\widetilde{\w}} \\
&= - \frac{\partial \w}{\partial\widetilde{\w}} \nabla_{\w} L
\end{align*}
so the preconditioner is $\frac{\partial \w}{\partial\widetilde{\w}}$.

For Riemannian GD the update
\begin{align*}
\dot{\w} &= -\bm{\Gamma}_{\w}\nabla_{\w}L
\end{align*}
immediately implies the preconditioner is $\bm{\Gamma}_{\w}^{-1}$.
\end{proof}

We now analyze the implications of the theorem for EGU.
EGU is defined by the mirror map (applied componentwise):
\begin{equation*}
f(\w) = \log \w
\end{equation*}
This implies 
\begin{equation}
\label{e:EGU-preconditioner}
\frac{\partial\w}{\partial\widetilde{\w}} = {\rm Diag}(\w)
\end{equation}
Now consider the reparameterization (applied componentwise):
\begin{equation*}
\widehat{\w} = 2\sqrt{\w}
\end{equation*}
This implies 
\begin{equation*}
\frac{\partial\w}{\partial\widehat{\w}} = \tfrac{1}{2}{\rm Diag}(\widehat{\w})
\end{equation*}
and therefore 
\begin{equation*}
\frac{\partial\w}{\partial\widehat{\w}} \left(\frac{\partial\w}{\partial\widehat{\w}}\right)^\top = \frac{\partial\w}{\partial\widetilde{\w}}
\end{equation*}
as in Theorem~\ref{t:equivalence}. This means continuous-time EGU is equivalent to gradient flow on the spindly network when the two weights are set to be equal, i.e.\ $w_i = u_i^2$ instead of $w_i = u_i v_i$ (note that if $\u$ and $\v$ are initialized equally then they will remain equal) \citep{spindly}.

Using \eqref{e:EGU-preconditioner}, Theorem \ref{t:equivalence} also implies EGU is equivalent to Riemannian GD with metric
\begin{equation}
\label{e:EGU-metric}
\bm{\Gamma}_{\w} = {\rm Diag}(\w)^{-1}
\end{equation}
We visualize this metric in Figure~\ref{fig:EGU-metric} for a 2d parameter space. This geometry that is implicit in EGU helps to explain how the algorithm encourages weight trajectories to stay near sparse solutions.

\begin{figure}[t!]
    \centering
    \begin{center}
    \subfigure[EGU]{\includegraphics[width=0.47\linewidth]{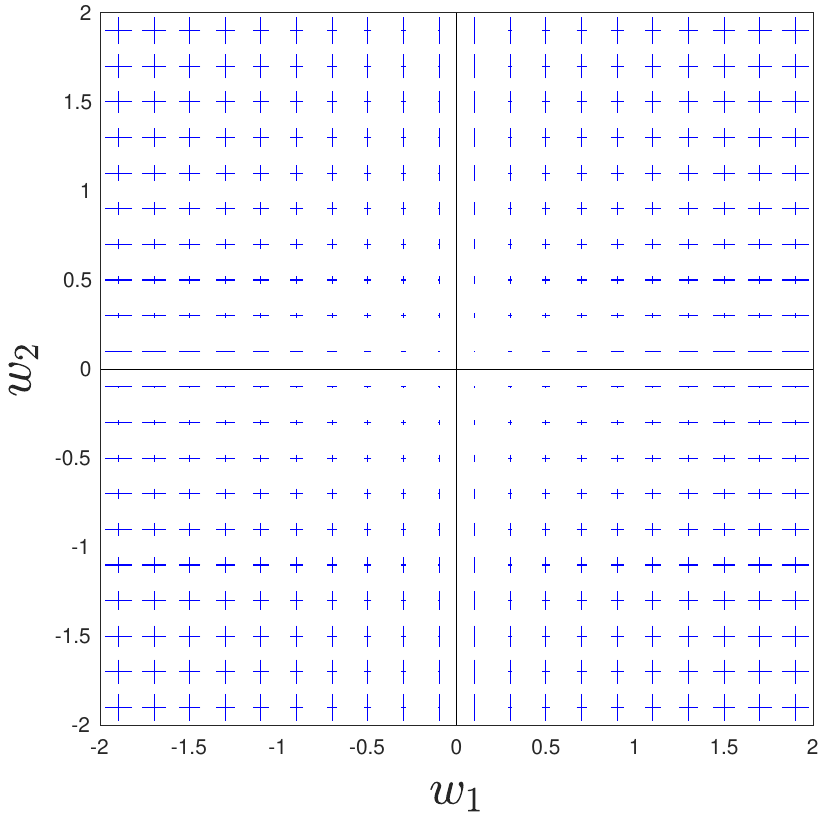}} 
    \quad
    \subfigure[GD]{\includegraphics[width=0.47\linewidth]{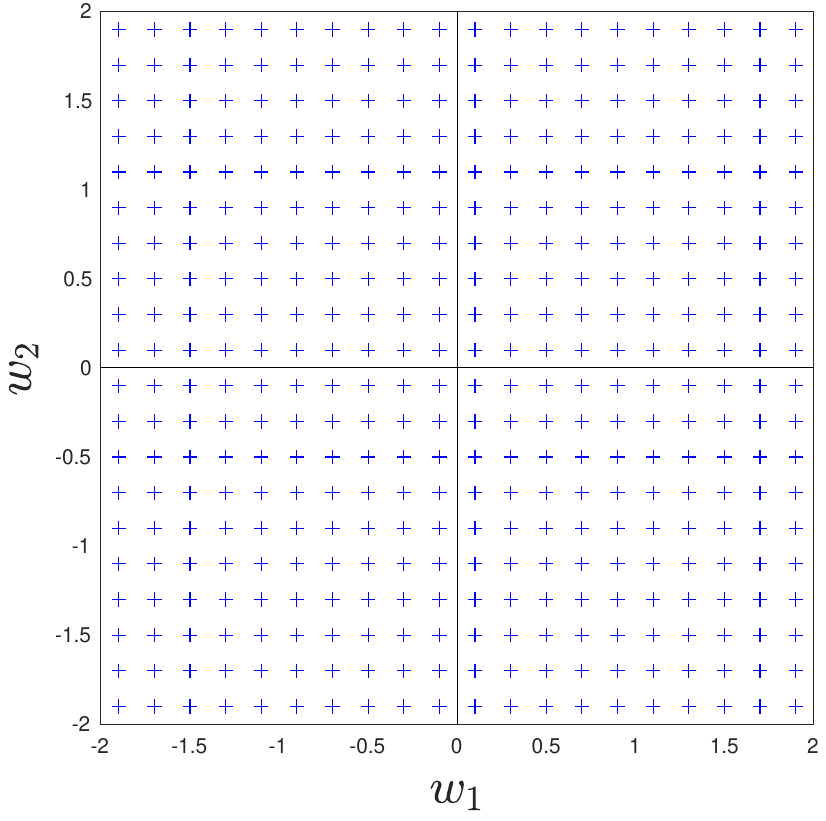}} 
    \end{center}
\caption{(a): Metric in \eqref{e:EGU-metric} for the Riemannian GD interpretation of EGU. (b): For comparison, the Euclidean metric implicit in GD, which is uniform and in particular rotationally symmetric. Blue lines indicate intervals of constant distance according to the respective metric.}
\label{fig:EGU-metric}
\end{figure}

\section{Trajectory derivations}
\label{a:trajectories}

We analyze several algorithms on the regression problem of Sections \ref{s:lower_bounds} and \ref{s:upper_bounds}.
The mean loss over the training set is
\begin{equation*}
L(\w) = \frac{1}{md} \Vert \X\w - \y \Vert^2
\end{equation*}
Using $\X^\top\X=md\I$, the gradient is
\begin{align*}
\nabla_{\w}L &= 2(\w - \w^{\rm LS}) \\
\w^{\rm LS} &= \frac{1}{md} \X^\top\y
\end{align*}
where $\w^{\rm LS}$ is the linear least-squares solution to which all considered algorithms converge.

The continuous EGU update can be written as 
\begin{equation*}
\dot{w}_i = -w_i \nabla_{w_i} L = 2w_i (w^{\rm LS}_i - w_i)
\end{equation*}
The trajectory in \eqref{e:EGU-trajectory} can be directly verified to satisfy this PDE, with the initial condition $w_i(0)$ satisfied by setting
\begin{equation*}
c_i = \tanh^{-1} \left( \frac{2 w_i(0)}{w^{\rm LS}_i} - 1 \right)
\end{equation*}

The continuous EGU$\pm$ update can be written as
\citep{regretcont}
\begin{equation*}
\dot{w}_i = -\sqrt{w_i^2+1}\, \nabla_{w_i} L = 2\sqrt{w_i^2+1} \, (w^{\rm LS}_i - w_i)
\end{equation*}
The trajectory in \eqref{e:EGUpm-trajectory} can be directly verified to satisfy this PDE, with the initial condition $w_i(0)$ satisfied by setting
\begin{align*}
\tau_i &= \sinh^{-1}\left( \frac{1 + w^{\rm LS}_i w_i(0)}{w_i(0) - w^{\rm LS}_i} \right) + 2t\xi \sqrt{(w^{\rm LS}_i)^2 + 1} \\
\xi &= {\rm sign} \left( w_i(0) - w^{\rm LS}_i \right)
\end{align*}

Primed gradient flow amounts to GD under the reparameterization $\widehat{\w} = {\rm Diag}(\w^{\rm LS})^{-1} \w$ and so by Theorem~\ref{t:equivalence} the update can be written as
\begin{align*}
\dot{w}_i = -\left(w^{\rm LS}_i\right)^2 \nabla_{w_i} L = 2 \left(w^{\rm LS}_i\right)^2 (w_i^{\rm LS} - w_i)
\end{align*}
The trajectory in \eqref{e:primeGD-trajectory} can be directly verified to satisfy this PDE and the initial condition $w_i(0)$.

{\bf Adagrad}: Continuous-time Adagrad can be written as
\begin{align*}
\dot{w}_i &= -G_i^{-1/2} \nabla_{w_i} L = 2 G_i^{-1/2} (w^{\rm LS}_i - w_i) \\
\dot{G}_i &= \beta (\nabla_{w_i} L)^2 = 4 \beta (w^{\rm LS}_i - w_i)^2
\end{align*}
with preconditioner learning rate $\beta$ and $G_i(0) = \varepsilon$ a stability parameter. 

To solve these coupled PDEs we begin by defining 
\begin{equation*}
\delta_{i}=w_{i}-w_{i}^{{\rm LS}}
\end{equation*}
which leads to
\begin{align}
\dot{\delta}_{i}	&= -2G_{i}^{-1/2}\delta_{i} \nonumber\\
\dot{G}_{i}	&=4\beta\delta_{i}^{2}
\label{e:Gdot-delta2}
\end{align}
Combining these two equations yields
\begin{align*}
\ddot{G}_{i}	&=8\beta\delta_{i}\dot{\delta_{i}} \\
	&=-4G_{i}^{-1/2}4\beta\delta_{i}^{2} \\
	&=-4G_{i}^{-1/2}\dot{G}_{i}
\end{align*}	
which has solution
\begin{align}
G_{i}	&=\frac{16}{k_{i}^{2}}\left(W\left(-e^{-k_{i}\left(t+\ell_{i}\right)}\right)+1\right)^{2} \nonumber\\
\dot{G}_{i}	&=-\frac{32}{k_{i}}W\left(-e^{-k_{i}\left(t+\ell_{i}\right)}\right)
\label{e:Gdot-k-ell}
\end{align}	
Using $G_{i}\left(0\right)=\varepsilon$ and $\dot{G}_{i}\left(0\right)=4\beta\left(w_{i}^{{\rm LS}}-w_{i}\left(0\right)\right)^{2}$ allows to solve for the constants $k_i$ and $\ell_i$:
\begin{align*}
k_i &= \frac{8}{\beta(w^{\rm LS}_i - w_i(0))^2 + 2\sqrt{\varepsilon}} \\
\ell_i &= \frac{1}{k_i}-\frac{1}{k_i}\log\left(1-\frac{k_i\sqrt{\varepsilon}}{4}\right)-\frac{\sqrt{\varepsilon}}{4}
\end{align*}
Substituting \eqref{e:Gdot-delta2} and \eqref{e:Gdot-k-ell} gives the corresponding expression for $w_{i}$, matching \eqref{e:Adagrad-trajectory}:
\begin{align*}
w_{i}	&=w_{i}^{{\rm LS}}+\delta_{i} \\
	&=w_{i}^{{\rm LS}}+{\rm sign}\left(w_{i}\left(0\right)-w_{i}^{{\rm LS}}\right)\frac{\sqrt{\dot{G}_{i}}}{2\sqrt{\beta}} \\
	&=w_{i}^{{\rm LS}}-{\rm sign}\left(w_{i}^{{\rm LS}}-w_{i}\left(0\right)\right)\sqrt{-\frac{8}{\beta k_{i}}W\left(-e^{-k_{i}\left(t+\ell_{i}\right)}\right)}
\end{align*}

{\bf Incremental priming and Burg MD}: 
We also consider an incremental version of priming where the learned weights are continuously transferred into the priming vector rather than only once at the end of a pre-training phase.
Specifically, we begin with the predictive model
\begin{equation*}
\yh_t(\x) = (\x \odot \p_t)^\top \w_t
\end{equation*}
where $\p_t$ is the priming vector, $\w_t$ is the weight vector, and $t$ indexes iterations of the learning algorithm. The idea is to maintain $\w_t = \bm{1}$, transferring each update of $\w$ immediately into $\p$. Specifically, at each time step we make a provisional GD update
\begin{align*}
\widetilde{\w}_{t+1} &= \w_t - \eta \nabla_{\w_t} L(\predy_t(\X), \y) \\
&= 1 + \eta {\rm Diag}(\p_t) \X^\top (\y - \X \p_t) 
\end{align*}
where the second line uses the inductive assumption $\w_t=\bm{1}$. This assumption holds because we immediately transfer the provisional update into $\p$:
\begin{align*}
\p_{t+1} &= \p_t \odot \widetilde{\w}_{t+1} \\
\w_{t+1} &= \bm{1}
\end{align*}
This transfer leaves predictions at step $t+1$ unchanged because
$(\x\odot\p_t)^\top\widetilde{\w}_{t+1} = (\x\odot\p_{t+1})^\top \w_{t+1}$ for all $\x$,
but it leads the learning on step $t+1$ to contribute immediately to priming on future steps. 

We now simplify the model by dropping the inconsequential $\w$ and directly computing the update for $\p$:
\begin{align*}
\yh_t(\x) &= \x^\top \p_t \\
\p_{t+1} &= \p_t \odot (1 + \eta {\rm Diag}(\p_t) \X^\top (\y - \X \p_t) ) \\
&= \p_t - \eta {\rm Diag}(\p_t)^2 \, \nabla_{\p_t} L(\predy_t(\X), \y) 
\end{align*}
Treating $\p$ as the weight vector, we have a
preconditioned GD algorithm which, by
Theorem~\ref{t:equivalence}, is equivalent in the
continuous-time case to MD with mirror map $f(\p) =
-\p^{-1}$ (componentwise). This corresponds to Burg MD
\citep{burg-regr,regretcont}. 
(In the discrete time case, incremental priming corresponds
to the ``dual update'' of Burg MD \citep{jagota}, $\p_{t+1} = \p_t - \eta \nabla_{f(\p_t)} L$ instead of $f(\p_{t+1}) = f(\p_t) - \eta \nabla_{\p_t} L$.) 

Changing notation from $\p$ back to $\w$, we can write the update as
\begin{equation*}
\dot{w}_i = -w_i^2 \nabla_{w_i} L = 2w_i^2 (w^{\rm LS}_i - w_i)
\end{equation*}
This has the solution
\begin{align*}
w_i(t) &= \frac{w^{\rm LS}_i}{W\left(\exp\left(-2\left(w^{\rm LS}_i\right)^2(t-b_i)\right)\right)+1} \\
b_i &=\frac{1}{2\left(w_{i}^{{\rm LS}}\right)^{2}}\log W^{-1}\left(\frac{w_{i}^{{\rm LS}}}{w_{i}\left(0\right)}-1\right)
\end{align*}
where $W$ is Lambert's W function as above. The solution can be verified by substituting it back into the PDE.

\section{Anisotropic covariance}

Theorem \ref{thm:lower_bound} and the analytic trajectory solutions in Section~\ref{s:trajectories} assume spherical input covariance, meaning $\X^{\top}\X$
is proportional to the identity. When it is not, rotationally invariant
algorithms can produce trajectories that curve toward sparse solutions
under certain circumstances (Figure~\ref{fig:unbalanced}a). Nevertheless, Theorem 1 implies a rotationally
invariant algorithm cannot perform better on sparse over non-sparse
problems with a rotationally symmetric input distribution, i.e.\ when averaged over all rotations of $\X^\top\X$.

For a rotationally invariant algorithm on a linear problem, rotating
the problem rotates the entire weight trajectory: $X\to XU^{\top}$
implies $w(t)\to Uw(t)$. This allows us to understand the algorithm's
behavior with a fixed sparse target $v$ and rotated input $XU^{\top}$
by examining its behavior with a rotated target $Uv$ and fixed input
$X$. We illustrate this in Figure~\ref{fig:unbalanced} for a 2-dimensional
problem with $X=H\,{\rm Diag}\left(2,1\right)$ so that $X^{\top}X$
has condition number 4. On the left, the first principal component of $\X^\top\X$ is aligned with the sparse target. GD and EGU$\pm$ both produce trajectories that bend
toward the target, but for different reasons: GD's sparsity
bias depends on $\X$ while EGU$\pm$'s does not. This is seen in the right figure where the input is rotated. Rotation invariance
of GD implies its trajectories also rotate, so that it no longer learns efficiently. Thus GD has no sparsity advantage
when averaging over all rotated inputs. In contrast, EGU$\pm$ shows
a sparsity bias in all cases. Because it is not rotationally invariant,
rotating the problem does not rotate its trajectories (they are altered but to a much lesser degree).

\begin{figure}[t!]
    \centering
    \begin{center}
    \subfigure[Unbalanced covariance]{\includegraphics[width=0.47\linewidth]{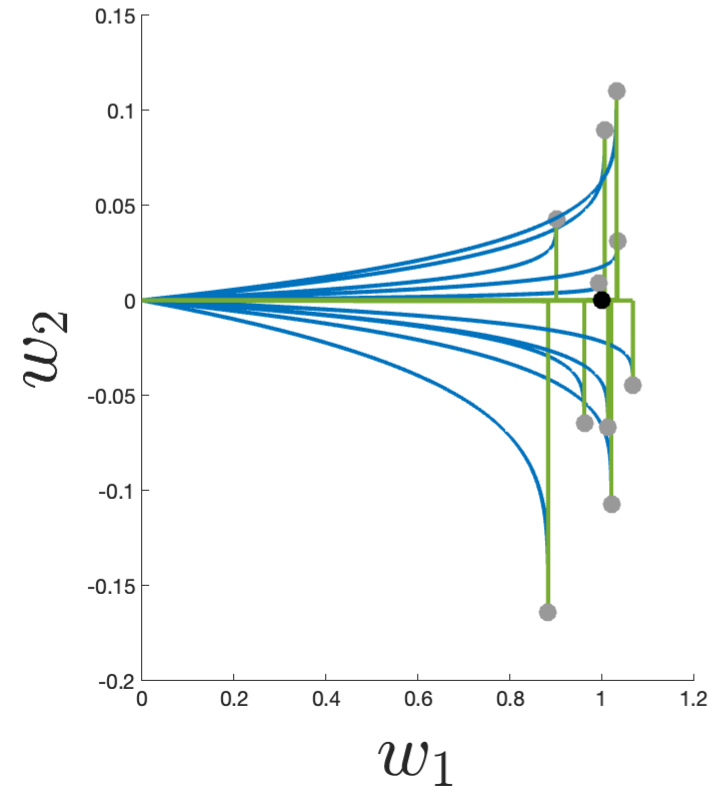}} 
    \subfigure[Unbalanced covariance, rotated]{\includegraphics[width=0.47\linewidth]{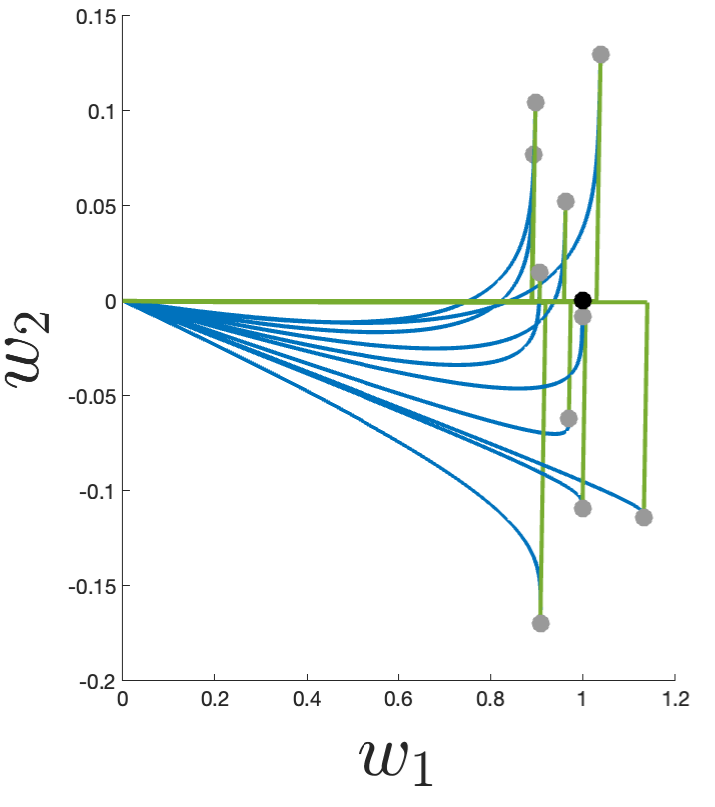}} 
    \end{center}
    \caption{
    GD produces curved trajectories (blue) when the covariance $\X^\top\X$ is nonspherical. 
    This can speed learning then the first principal component is aligned with the target as in (a).
    However, GD and other rotationally invariant algorithms cannot produce a systematic bias toward sparsity. When the input distribution is rotated as in (b), the trajectories rotate as well such that they no longer learn the sparse target efficiently. In contrast, non-rotation invariant algorithms such as EGU$\pm$ (green) can learn sparse targets efficiently under any rotation of the input.
    }
\label{fig:unbalanced}
\end{figure}

\section{Fashion MNIST experiment details}
\label{a:fMNIST}
In our experiments, we use a constant learning rate (which
we tune for each case). We use the full batch of $60000$
training examples and train each network for $5000$ epochs.
We first provide some visualization of the weights for the
noisy case where each example is augmented with unifrom noise.
Figure~\ref{fig:image_noise} shows a subset of the weights
for each network where the top slice corresponds to the
image feature weights and the bottom slice corresponds to
the noise feature weights. For the spindly network, the
average maximum absolute value of the effective weights (i.e., the product of the two weights within each spindle) for each input
neuron is $0.0182$ for the image weights and
$0.0025$ for
the noise features. The difference is less drastic for the
fully-connected network, where the values are
$0.0627$ and $0.0568$, respectively.

Next, we show the results when adding extra one-hot
embeddings of the labels as features.
Figure~\ref{fig:image_noise_labels} shows a subset of the
image and label weights for each network, along with the
weights corresponding to the labels at the bottom. The
spindly network assigns relatively larger weights to the
label features. The average maximum absolute value of the
weights for each neuron is $0.7834$ for the label
weights, whereas the image and noise weights have values
$0.0057$ and $0.0025$, respectively. Again,
the difference between the label weights and the rest is
less prominent for the fully-connected network: The average
maximum absolute values are $0.4213$ for label
weights and $0.0604$ and $0.0584$ for the image and noise weights, respectively.

\begin{figure}[t!]
    \centering
    \begin{center}
    \subfigure[Fully-connected]{\includegraphics[width=0.47\linewidth]{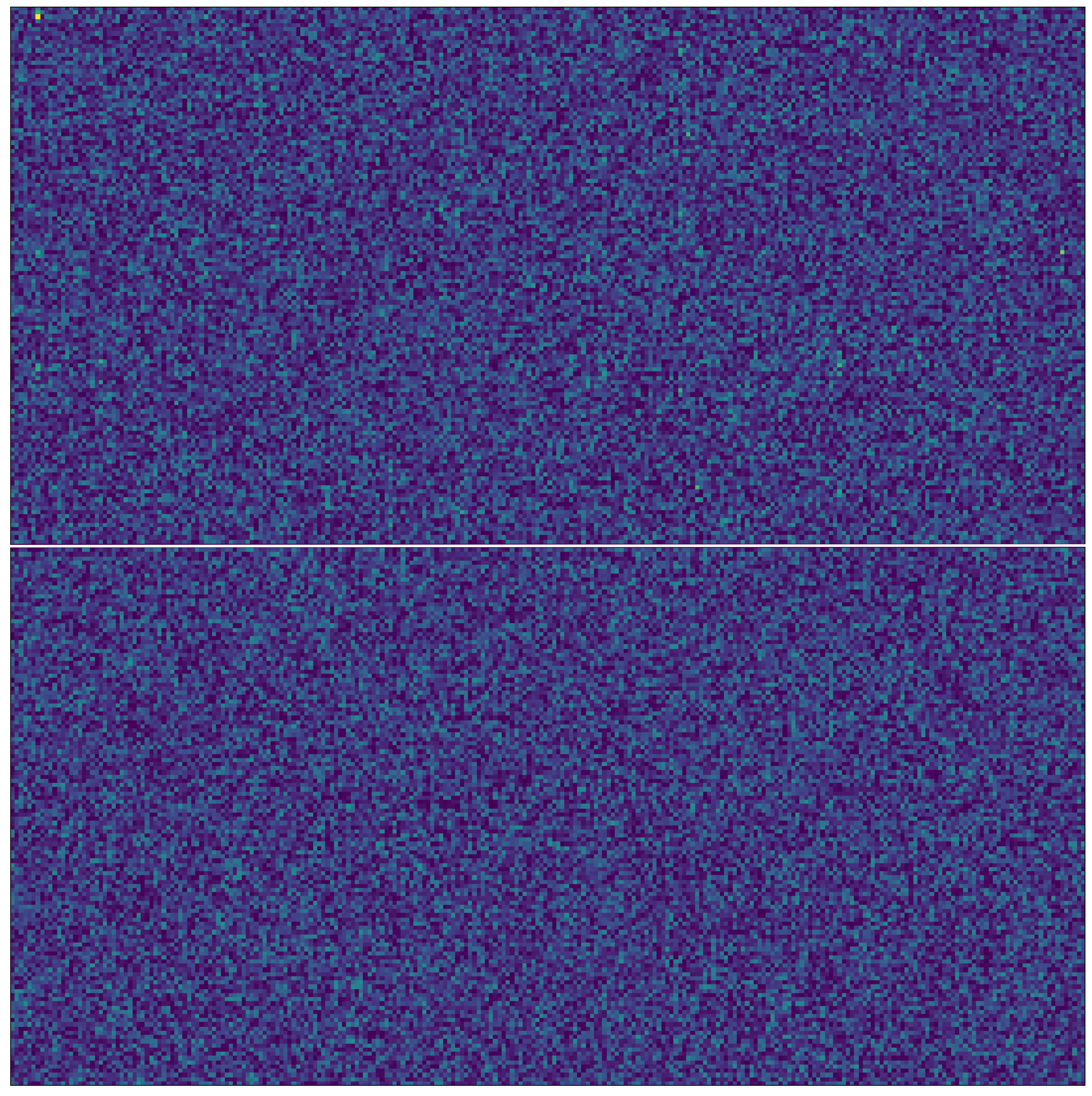}} 
    \subfigure[Spindly]{\includegraphics[width=0.47\linewidth]{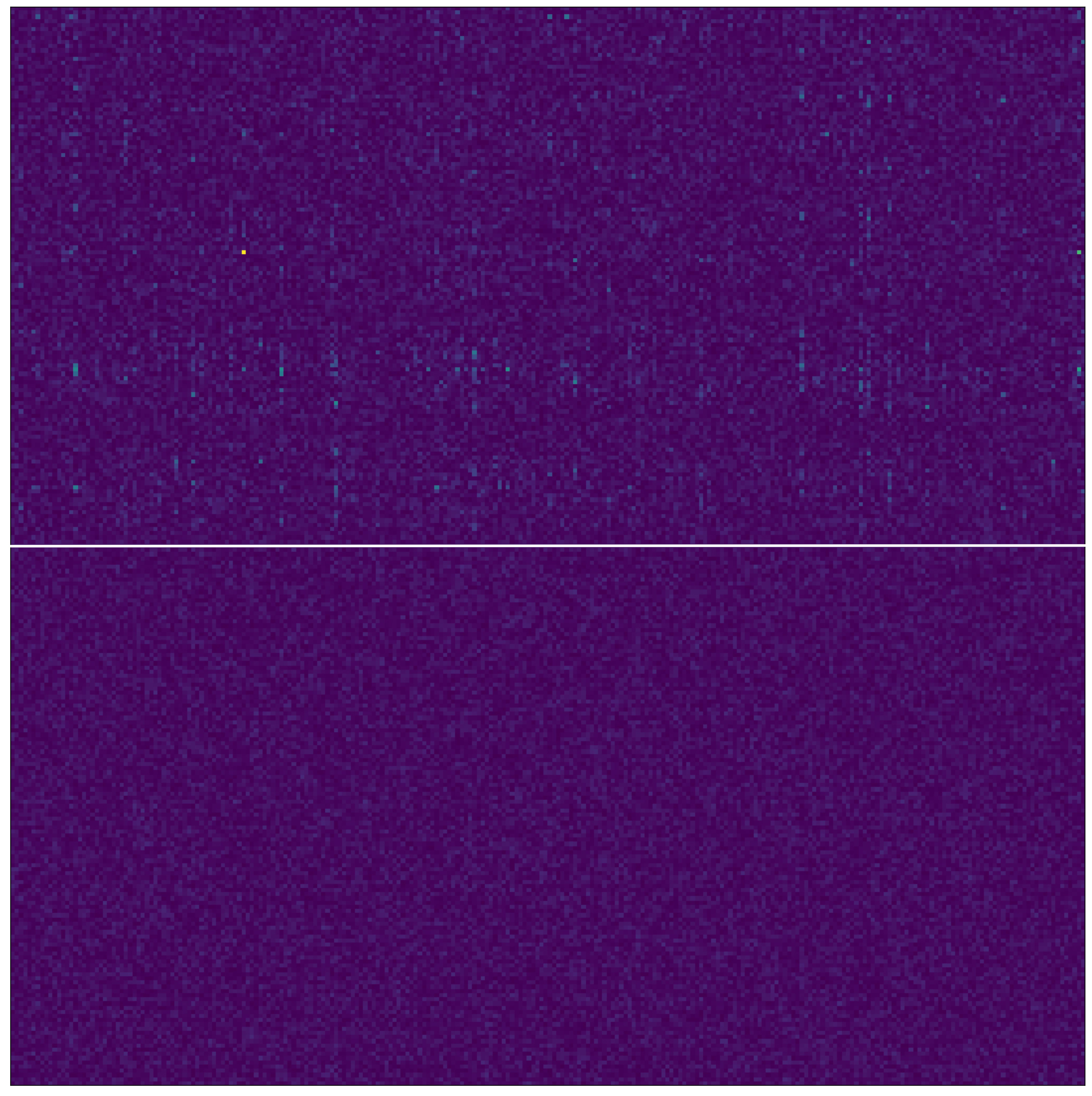}}
    \end{center}
\caption{The weights of the first layer when trained with images augmented with noise. The top slice corresponds to the image feature weights and the bottom slice corresponds to the noise feature weights.}\label{fig:image_noise}
\end{figure}

\begin{figure}[t!]
    \begin{center}
    \subfigure[Fully-connected]{\includegraphics[width=0.47\linewidth]{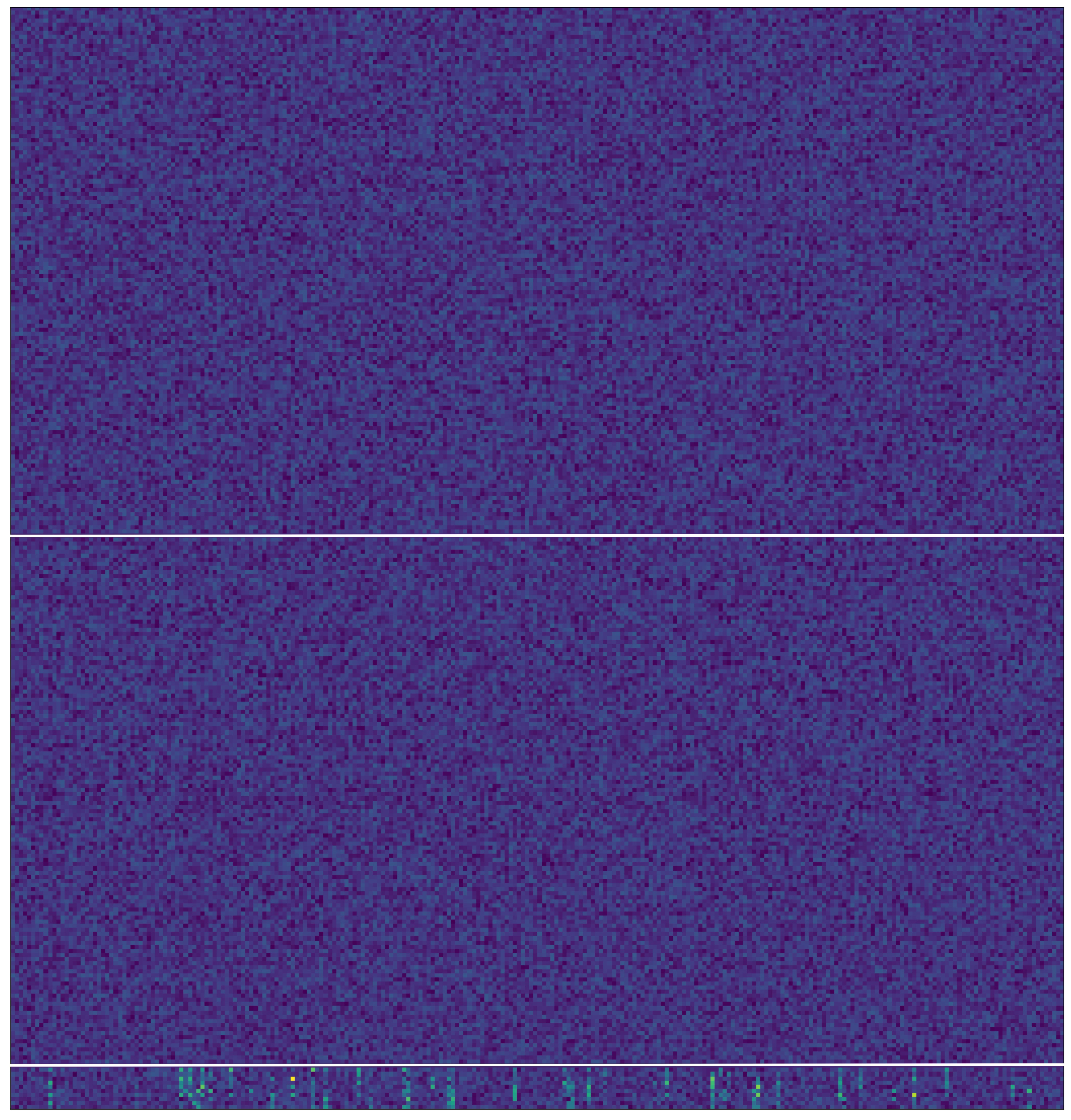}} 
    \subfigure[Spindly]{\includegraphics[width=0.47\linewidth]{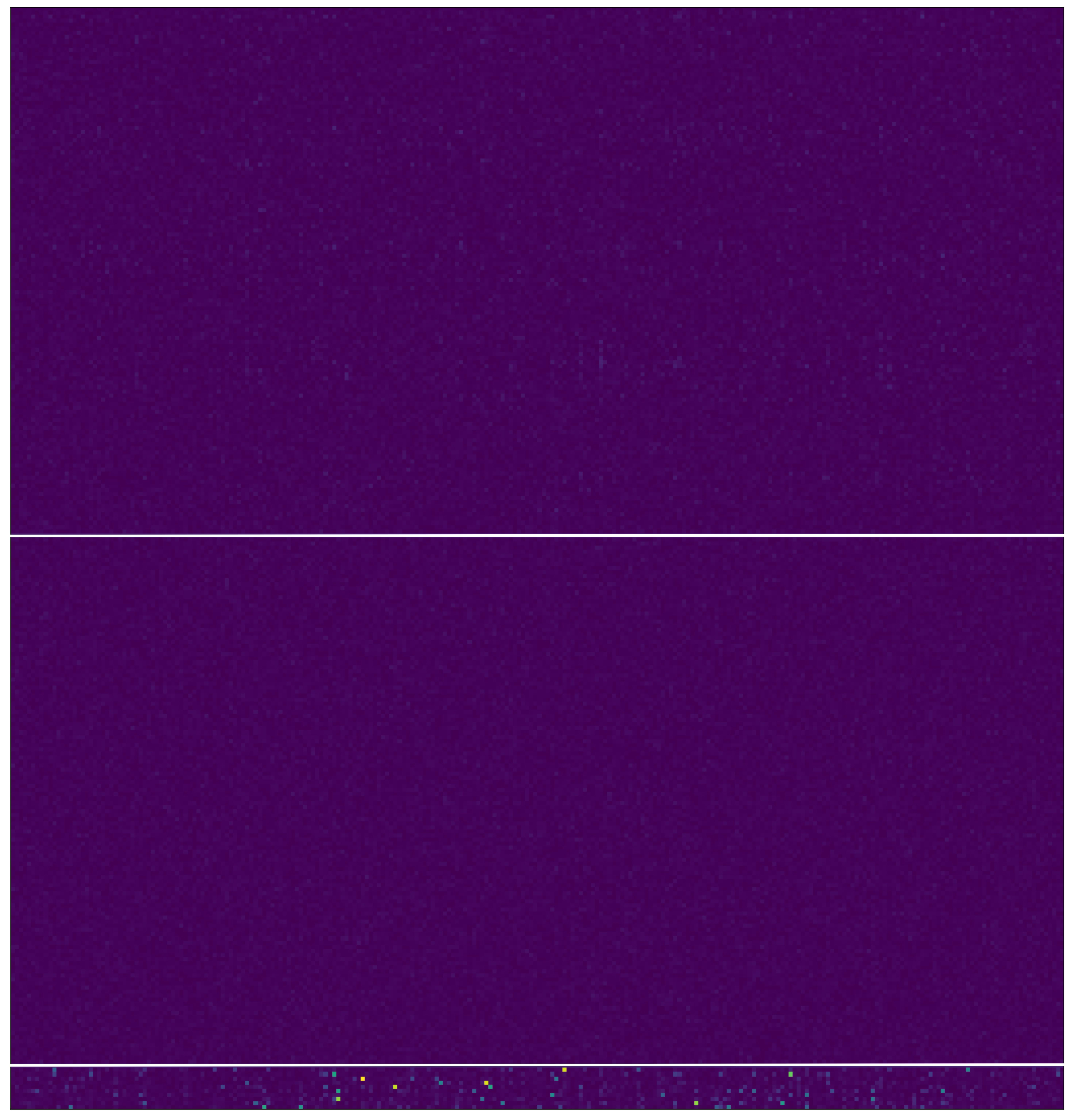}}
    \end{center}
\caption{The weights of the first layer when trained with images augmented with noise and one-hot representation of the labels. The top slice corresponds to the image feature weights, the middle slice corresponds to the noise feature weights, and the bottom corresponds to the label weights.}\label{fig:image_noise_labels}
\end{figure}

In summary, the network with the fully connected input
layer is more easily confused by noisy features.
When noisy and informative features are added then
the network with a spindlified input layer finds a model
that almost entirely relies on the informative features 
whereas the net with the fully connected input layer still makes
use of original and noisy features.

{\bf More network details:}
Three layer fully connected net with RELU transfer
functions, 10 outputs, square loss, trained with batch
gradient descent for 5000 epochs, batch size 60000.

Number of weights in the net with the fully connected input
layer: input size 784, 256 nodes on the first
two hidden layers, 10 outputs, for a total of 268.000
weights. Adding a copy of noisy input features almost
doubles the number of weights. The number of informative
features is $10*256=2560$.
It is important to note that the number of examples is much
larger than the input dimension but much smaller than the
total number of weights.
\end{document}

%% file: defs.tex
\newcommand*{\mycirc}{\mathrel{\mathsmaller{\mathsmaller{\circ}}}}
\renewcommand{\mathring}[1]{\overset{\mycirc}{#1}}
\DeclareMathOperator{\EE}{\mathbb{E}} 
\newcommand{\predy}{\widehat{\y}}
\newcommand{\predw}{\widehat{\w}}

\newcommand{\SSigma}{\bm{\Sigma}}
\newcommand\sbullet[1][.5]{\mathbin{\vcenter{\hbox{\scalebox{#1}{$\bullet$}}}}}
\renewcommand{\dot}[1]{    {\stackrel{\scriptscriptstyle \sbullet[0.42]}{#1}}    }

\DeclareMathOperator*{\argmax}{\mathop{\mathrm{argmax}}}
\DeclareMathOperator*{\argmin}{\mathop{\mathrm{argmin}}}

\newcommand{\x}{\bm{x}}
\renewcommand{\H}{\bm{H}}

\newcommand{\M}{\bm{M}}

\renewcommand{\u}{\bm{u}}

\newcommand{\e}{\bm{e}}
\newcommand{\h}{\bm{h}}
\newcommand{\p}{\bm{p}}
\newcommand{\s}{\bm{s}}
\renewcommand{\v}{\bm{v}}
\newcommand{\y}{\bm{y}}
\newcommand{\X}{\bm{X}}

\newcommand{\Z}{\bm{Z}}
\newcommand{\U}{\bm{U}}
\newcommand{\V}{\bm{V}}
\newcommand{\w}{\bm{w}}
\newcommand{\W}{\bm{W}}
\newcommand{\I}{\bm{I}}

\DeclareMathOperator{\tr}{\mathop{\mathrm{tr}}}

\newcommand{\yh}{\widehat{y}}

\newcommand{\Red}[1]{\color{red}{#1}\color{black}{}}

\newcommand{\Green}[1]{{\color{green}{#1}}}

\newcommand{\Blue}[1]{\textcolor{blue}{#1}}

\newcommand{\Brown} [1]{{\color{brown} {#1}}}

\definecolor{darkgreen}{rgb}{0.09, 0.45, 0.27}

\edef\polishl{\l}

\newcommand{\1}{\bm{1}}
\newcommand{\0}{\bm{0}}
\newcommand{\avgxi}{\bar{\bm{\xi}}}
\newcommand{\emploss}{L}
\newcommand{\Xte}{\x_{\rm{te}}}
\newcommand{\Xfull}{\tilde{\bm{X}}}
\newcommand{\yfull}{\tilde{\y}}
\newcommand{\yte}{y_{\rm{te}}}

%% file: main.bbl
\begin{thebibliography}{32}
\providecommand{\natexlab}[1]{#1}
\providecommand{\url}[1]{\texttt{#1}}
\expandafter\ifx\csname urlstyle\endcsname\relax
  \providecommand{\doi}[1]{doi: #1}\else
  \providecommand{\doi}{doi: \begingroup \urlstyle{rm}\Url}\fi

\bibitem[Abbe and Boix-Adsera(2022)]{abbe}
Emmanuel Abbe and Enric Boix-Adsera.
\newblock On the non-universality of deep learning: quantifying the cost of
  symmetry.
\newblock In \emph{Advances in Neural Information Processing Systems},
  volume~35, pages 17188--17201. Curran Associates, Inc., 2022.

\bibitem[Abdulkadirov et~al.(2023)Abdulkadirov, Lyakhov, and
  Nagornov]{optsurvey}
Ruslan Abdulkadirov, Pavel Lyakhov, and Nikolay Nagornov.
\newblock Survey of optimization algorithms in modern neural networks, 04 2023.

\bibitem[Amari and Douglas(1998)]{Amari98}
S.~Amari and S.C. Douglas.
\newblock Why natural gradient?
\newblock In \emph{Proceedings of the 1998 IEEE International Conference on
  Acoustics, Speech and Signal Processing, ICASSP '98 (Cat. No.98CH36181)},
  volume~2, pages 1213--1216 vol.2, 1998.

\bibitem[Amari(1998)]{amari1998natural}
Shun-Ichi Amari.
\newblock Natural gradient works efficiently in learning.
\newblock \emph{Neural computation}, 10\penalty0 (2):\penalty0 251--276, 1998.

\bibitem[Amid and Warmuth(2020)]{regretcont}
Ehsan Amid and Manfred~K. Warmuth.
\newblock Reparameterizing mirror descent as gradient descent.
\newblock In \emph{Proceedings of Advances in Neural Information Processing
  Systems}, volume~33, pages 8430--8439, 2020.

\bibitem[Berger(1985)]{Berger}
James~O. Berger.
\newblock \emph{Statistical decision theory and Bayesian analysis}.
\newblock Springer, 1985.

\bibitem[Bonnabel(2013)]{bonnabel2013stochastic}
Silvere Bonnabel.
\newblock Stochastic gradient descent on riemannian manifolds.
\newblock \emph{IEEE Transactions on Automatic Control}, 58\penalty0
  (9):\penalty0 2217--2229, 2013.

\bibitem[Cesa-Bianchi et~al.(2007)Cesa-Bianchi, Mansour, and Stoltz]{prod}
Nicolo Cesa-Bianchi, Yishay Mansour, and Gilles Stoltz.
\newblock Improved second-order bounds for prediction with expert advice.
\newblock \emph{Machine Learning}, 66\penalty0 (2-3):\penalty0 321--352, 2007.

\bibitem[Chang et~al.(2023)Chang, Dur{\`a}n-Mart{\'\i}n, Shestopaloff, Jones,
  and Murphy]{chang2023low}
Peter Chang, Gerardo Dur{\`a}n-Mart{\'\i}n, Alexander~Y Shestopaloff, Matt
  Jones, and Kevin Murphy.
\newblock Low-rank extended kalman filtering for online learning of neural
  networks from streaming data.
\newblock \emph{arXiv preprint arXiv:2305.19535}, 2023.

\bibitem[Dicker(2016)]{Dickel}
Lee~H. Dicker.
\newblock Ridge regression and asymptotic minimax estimation over spheres of
  growing dimension.
\newblock \emph{Bernoulli}, 22\penalty0 (1):\penalty0 1--37, 2016.

\bibitem[Duchi et~al.(2011{\natexlab{a}})Duchi, Hazan, and Singer]{adagrad}
John Duchi, Elad Hazan, and Yoram Singer.
\newblock Adaptive subgradient methods for online learning and stochastic
  optimization.
\newblock \emph{Journal of Machine Learning Research}, 12\penalty0
  (Jul):\penalty0 2121--2159, 2011{\natexlab{a}}.

\bibitem[Duchi et~al.(2011{\natexlab{b}})Duchi, Hazan, and
  Singer]{duchi2011adaptive}
John Duchi, Elad Hazan, and Yoram Singer.
\newblock Adaptive subgradient methods for online learning and stochastic
  optimization.
\newblock \emph{Journal of machine learning research}, 12\penalty0 (7),
  2011{\natexlab{b}}.

\bibitem[Gerchinovitz et~al.(2020)Gerchinovitz, M{\'e}nard, and
  Stoltz]{Gerchinovitz_etal2020}
S{\'e}bastien Gerchinovitz, Pierre M{\'e}nard, and Gilles Stoltz.
\newblock {Fano’s Inequality for Random Variables}.
\newblock \emph{Statistical Science}, 35\penalty0 (2):\penalty0 178 -- 201,
  2020.

\bibitem[Gunasekar et~al.(2017)Gunasekar, Woodworth, Bhojanapalli, Neyshabur,
  and Srebro]{srebro1}
Suriya Gunasekar, Blake~E. Woodworth, Srinadh Bhojanapalli, Behnam Neyshabur,
  and Nati Srebro.
\newblock Implicit regularization in matrix factorization.
\newblock In \emph{Proceedings of Advances in Neural Information Processing
  Systems (NeurIPS)}, pages 6151--6159, 2017.

\bibitem[Jain et~al.(2007)Jain, Kulis, and Dhillon]{burg-regr}
Prateek Jain, Brian~J. Kulis, and Inderjit~S. Dhillon.
\newblock Online linear regression using burg entropy.
\newblock Technical Report TR-07-08, The University of Texas at Austin, feb
  2007.

\bibitem[Kerekes et~al.(2021)Kerekes, Mészáros, and Huszár]{depthnat}
Anna Kerekes, Anna Mészáros, and Ferenc Huszár.
\newblock Depth without the magic: Inductive bias of natural gradient descent.
\newblock \emph{ArXiv}, abs/2111.11542, 2021.

\bibitem[Kingma and Ba(2015)]{adam}
Diederik Kingma and Jimmy Ba.
\newblock Adam: A method for stochastic optimization.
\newblock In \emph{International Conference on Learning Representations
  (ICLR)}, San Diega, CA, USA, 2015.

\bibitem[Kivinen and Warmuth(1997)]{eg}
Jyrki Kivinen and Manfred~K. Warmuth.
\newblock Exponentiated gradient versus gradient descent for linear predictors.
\newblock \emph{Information and Computation}, 132\penalty0 (1):\penalty0 1--63,
  1997.

\bibitem[Kivinen et~al.(1997)Kivinen, Warmuth, and Auer]{percwinn}
Jyrki Kivinen, Manfred~K. Warmuth, and Peter Auer.
\newblock The {P}erceptron algorithm versus {W}innow: linear versus logarithmic
  mistake bounds when few input variables are relevant.
\newblock \emph{Artificial Intelligence}, 97\penalty0 (1-2):\penalty0 325--343,
  1997.

\bibitem[Lambert et~al.(2022)Lambert, Bonnabel, and Bach]{lambert2022recursive}
Marc Lambert, Silvere Bonnabel, and Francis Bach.
\newblock The recursive variational gaussian approximation (r-vga).
\newblock \emph{Statistics and Computing}, 32\penalty0 (1):\penalty0 10, 2022.

\bibitem[Li et~al.(2021)Li, Zhang, and Arora]{arora-conv}
Zhiyuan Li, Yi~Zhang, and Sanjeev Arora.
\newblock Why are convolutional nets more sample-efficient than fully-connected
  nets?
\newblock In \emph{International Conference on Learning Representations}, 2021.

\bibitem[Marchand(1993)]{Marchand}
Eric Marchand.
\newblock Estimation of a multivariate mean with constraints on the norm.
\newblock \emph{The Canadian Journal of Statistics}, 21\penalty0 (4):\penalty0
  359--366, 1993.

\bibitem[Nemirovsky and Yudin(1983)]{MDref}
Arkadi Nemirovsky and David Yudin.
\newblock \emph{Problem Complexity and Method Efficiency in Optimization}.
\newblock John Wiley \& Sons, 1983.

\bibitem[Ng(2004)]{ng}
Andrew~Y. Ng.
\newblock Feature selection, $l_1$ vs. $l_2$ regularization, and rotational
  invariance.
\newblock In \emph{Proceedings of the 21-st International Conference on Machine
  Learning, Banff, Canada, July 4-8}. ACM New York, NY, 2004.

\bibitem[Rigollet and Hütter(2023)]{rigollet2023highdimensional}
Philippe Rigollet and Jan-Christian Hütter.
\newblock High-dimensional statistics, 2023.

\bibitem[Shamir(2018)]{distr-spec-shamir}
Ohad Shamir.
\newblock Distribution-specific hardness of learning neural networks.
\newblock \emph{Journal of Machine Learning Research}, 19\penalty0
  (32):\penalty0 1--29, 2018.

\bibitem[Vaskevicius et~al.(2019)Vaskevicius, Kanade, and Rebeschini]{optimal}
Tomas Vaskevicius, Varun Kanade, and Patrick Rebeschini.
\newblock Implicit regularization for optimal sparse recovery.
\newblock In \emph{Proceedings of Advances in Neural Information Processing
  Systems (NeurIPS)}, pages 2968--2979, 2019.

\bibitem[Warmuth and Amid(2023)]{priming}
Manfred~K. Warmuth and Ehsan Amid.
\newblock Open problem: Learning sparse linear concepts by priming the
  features.
\newblock In \emph{Proceedings of Thirty Sixth Conference on Learning Theory},
  volume 195 of \emph{Proceedings of Machine Learning Research}, pages
  5937--5942. PMLR, Jul 2023.

\bibitem[Warmuth and Jagota(1998)]{jagota}
Manfred~K. Warmuth and Arun Jagota.
\newblock Continuous and discrete time nonlinear gradient descent: relative
  loss bounds and convergence.
\newblock In R.~Greiner E.~Boros, editor, \emph{Electronic Proceedings of Fifth
  International Symposium on Artificial Intelligence and Mathematics}.
  Electronic,http://rutcor.rutgers.edu/$\tilde{~}$amai, 1998.

\bibitem[Warmuth and Vishwanathan(2005)]{span}
Manfred~K. Warmuth and S.V.N. Vishwanathan.
\newblock Leaving the span.
\newblock In \emph{Proceedings of the 18th Annual Conference on Learning Theory
  (COLT)}, pages 366--381, 2005.

\bibitem[Warmuth et~al.(2021)Warmuth, Kot{\l}owski, and Amid]{spindly}
Manfred~K. Warmuth, Wojciech Kot{\l}owski, and Ehsan Amid.
\newblock A case where a spindly two-layer linear network decisively
  outperforms any neural network with a fully connected input layer.
\newblock In \emph{32th International Conference on Algorithmic Learning Theory
  (ALT)}, volume 132, pages 1--32. PMLR, 2021.

\bibitem[Xiao et~al.(2017)Xiao, Rasul, and Vollgraf]{xiao2017fashionmnist}
Han Xiao, Kashif Rasul, and Roland Vollgraf.
\newblock Fashion-mnist: a novel image dataset for benchmarking machine
  learning algorithms, 2017.
\newblock URL \url{http://arxiv.org/abs/1708.07747}.
\newblock cite arxiv:1708.07747Comment: Dataset is freely available at
  https://github.com/zalandoresearch/fashion-mnist Benchmark is available at
  http://fashion-mnist.s3-website.eu-central-1.amazonaws.com/.

\end{thebibliography}
